%% file: main.tex
\PassOptionsToPackage{dvipsnames}{xcolor}
\documentclass[sigconf,natbib=true]{acmart}

\usepackage{algorithm}
\usepackage{algorithmic}

\usepackage{ccaption}
\usepackage{subcaption}

\usepackage{physics}
\usepackage{tikz}
\usepackage{mathdots}
\usepackage{cancel}
\usepackage{color}
\usepackage{siunitx}
\usepackage{array}
\usepackage{multirow}
\usepackage{gensymb}
\usepackage{tabularx}
\usepackage{extarrows}
\usepackage{booktabs}
\usetikzlibrary{fadings}
\usetikzlibrary{patterns}
\usetikzlibrary{shadows.blur}
\usetikzlibrary{shapes}
\usetikzlibrary{trees}

\usepackage{wrapfig}
\usepackage{svg}

\usepackage{pgfplots}
\DeclareUnicodeCharacter{2212}{−}
\usepgfplotslibrary{groupplots,dateplot}
\usetikzlibrary{patterns,shapes.arrows}
\pgfplotsset{compat=newest}

\usepackage{adjustbox}
\usepackage{todonotes}
\usepackage{hyperref}

\DeclareMathOperator*{\argmin}{arg\,min}

\acmConference[KDD Workshop on Machine Learning in Finance]{KDD Workshop on Machine Learning in Finance}{August 4, 2025}{Toronto, Canada}

\title[Attacks on Financial Reporting via Maximum Violated Multi-Objective]{Adversarial Machine Learning Attacks on Financial Reporting via Maximum Violated Multi-Objective Attack}

\author{Edward Raff}
\email{Raff\_Edward@bah.com}
\affiliation{%
  \institution{Booz Allen Hamilton}
  \city{Syracuse}
  \state{NY}
  \country{} 
}
\affiliation{%
  \institution{University of Maryland, Baltimore County}
  \city{Catonsville}
  \state{MD}
  \country{} 
}

\author{Karen Kukla}
\email{kakukla@syr.edu}
\affiliation{%
  \institution{Syracuse University}
  \city{Syracuse}
  \state{NY}
  \country{} 
}

\author{Michel Benaroch}
\email{mbenaroc@syr.edu}
\affiliation{%
  \institution{Syracuse University}
  \city{Syracuse}
  \state{NY}
  \country{} 
}

\author{Joseph Comprix}
\email{jjcompri@syr.edu}
\affiliation{%
  \institution{Syracuse University}
  \city{Syracuse}
  \state{NY}
  \country{} 
}

\setcopyright{none}

\begin{document}

\begin{abstract}
Bad actors, primarily distressed firms, have the incentive and desire to manipulate their financial reports to hide their distress and derive personal gains. As attackers, these firms are motivated by potentially millions of dollars and the availability of many publicly disclosed and used financial modeling frameworks. 
Existing attack methods do not work on this data due to anti-correlated objectives that must both be satisfied for the attacker to succeed. We introduce Maximum Violated  Multi-Objective (MVMO) attacks that adapt the attacker's search direction to find $20\times$ more satisfying attacks compared to standard attacks.
The result is that in $\approx50\%$ of cases, a company could inflate their earnings by 100-200\%, while simultaneously reducing their fraud scores by 15\%.  By working with lawyers and professional accountants, we ensure our threat model is realistic to how such frauds are performed in practice. 
\end{abstract}

\maketitle

\section{Introduction}
Given a target function $f(\cdot)$ that takes an input $\boldsymbol{x}$ and output $\hat{y}$, it has been found that it is surprisingly easy for an adversary $\mathcal{A}$ to produce a perturbed $\boldsymbol{\Tilde{x}}$ that is trivially similar to input $x$, such that $\|\boldsymbol{\Tilde{x}} - \boldsymbol{x}\| \leq \epsilon$, and yet $f(\boldsymbol{\Tilde{x}}) \neq \hat{y}$. This has been achieved primarily for deep learning algorithms in classification tasks, such as computer vision (predict the label) and natural language processing (predict the next token). 

\begin{figure}[!h]
    \centering
    \includegraphics[width=1.0\columnwidth]{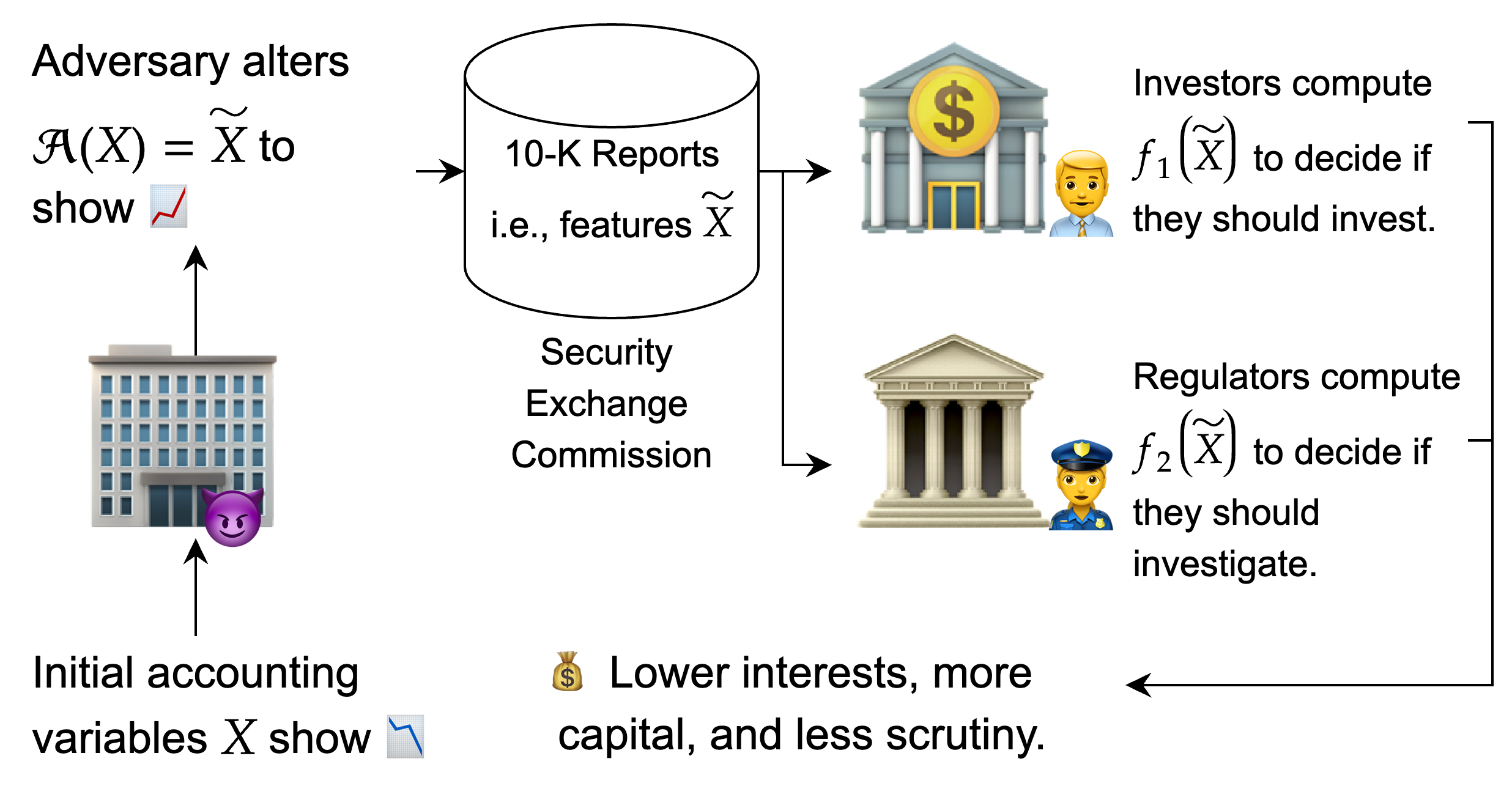}
    \caption{The overall threat model and scenario of this work. The adversary $\mathcal{A}$ is the (distressed) firm, which must file reports at a central database (run by the SEC). There are two models that need to be manipulated to achieve the adversarial firm's goals. These reports are used by investors to compute $f_1(\cdot)$ to decide if they should purchase shares or offer better loan terms (i.e., firm gains capital). Regulators also use these reports to run a fraud model $f_2(\cdot)$, to determine if investigation or other regulatory scrutiny of the firm is necessary.   }
    \label{fig:scenario}
\end{figure}

In this work, we seek to study a new problem space for adversarial machine learning (AML): \textit{fraudulent financial reporting}. This task has many differences from standard applications that make it interesting from a purely machine-learning perspective while simultaneously being highly relevant to real-world use (i.e., regulators need to know the extent of what is possible). Financial reporting fraud already occurs and is motivated by millions-to-billions of dollars in potential impact. 
Modeling frameworks used in this space are primarily linear in nature, making them a significantly different surface than most prior work, and the outputs $y$ are usually real-valued regression problems instead of classification. 

The decisions of investors, banks, and regulators can have significant impacts on a company in receiving new capital, loans, and scrutiny respectively. Each party makes these decisions in large part 
using models based on
financial ratios~\citep{altman_financial_1968}, which are computed from a financial statement called a 10-K. This 10-K is filed with the Security and Exchange Commission (SEC) every year and is produced \textit{by the company itself}. Thus, the company has the motivation and means to manipulate its financial reporting. 

To wit, our work is the first to study AML as applied to this problem. It is critical from a social perspective to understand the strengths and limitations of the current regulatory frameworks, i.e., Generally Accepted Accounting Principals (GAAP), with respect to potential risks as machine learning techniques become more widespread. Working with adversarial ML experts, two Certified Public Accountants (CPAs), and a lawyer, we build accurate and realistic threat models to apply AML methods to. Ultimately, we show positive and negative results for current GAAP-based accounting: in half of cases, adversaries can have extraordinary success in evading detection while inflating earnings, but in the other half of cases, they can only achieve one of these goals. 

We consider two sets of goals the adversary wishes to achieve: avoiding existing fraud models and inflating their reported earnings. Both goals are numeric targets $y$ with differing scales over orders of magnitude, making existing attack approaches ineffective. To demonstrate that an attack is still possible, we develop a novel paired numeric optimization that seeks to obtain an equal relative improvement in both goals. 
Working with domain expert accountants and lawyers, we develop a realistic threat model that defines an action-space for the threat model.

Our work uses two correlated metrics as exemplars: the earnings per share and a fraud model called M-score, which are prevalent in practical use. From a practical perspective, we show that the current environment is at risk of successful adversarial attacks where earnings are inflated by over 100\% while simultaneously reducing the fraud risk scores by 15\%. From an AML perspective, current techniques for adversarial regression do not succeed in this task, and we show how a multi-task regression attack can be constructed. Critically, it abides by the constraint that both earnings must go up, and the fraud score must go down, for the attack to be viable.

This article is organized as follows. First, we provide a brief primer on the relevant financial and accounting details needed to understand this work with related work in \S~\ref{sec:background}. Given this understanding, we can detail the threat model and its goals, attack strategy, and action space in \S~\ref{sec:threat_model}. Our novel MVMO attack is defined in \S\ref{sec:mvmo}. The results of our attacks are demonstrated on real-world financial data in \S~\ref{sec:results}, with ablation of adding a third objective to demonstrate MVMO's ability to handle additional targets.  Finally, we conclude in \S~\ref{sec:conclusion}.

This work is intrinsically interdisciplinary and thus covers a wide scope of material. We will first take a moment to review the context of the accounting background needed to understand the work and the accounting terminology that is necessary. 

\subsection{Accounting Context} \label{sec:background}

The set of players and goals in our work is outlined in \autoref{fig:scenario}.
There are two sets of actors in this work. \textit{First is the company (a.k.a. the firm)} that may engage in fraudulent reporting of its financial status, which tends to occur when a firm's financial health degrades~\citep{stolowy_accounts_2004,rosner_earnings_2003}. All publicly traded companies (in the United States, though most nations have a counterpart) must file a set of financial documents every fiscal year called a \textit{10-K}. A 10-K is freely available to anyone online and contains a balance sheet and income statement. These documents summarize the current net value of the firm and how the firm's net value changed (via revenue and costs) over the past year. \textit{Second are the investors (and regulators/lenders)}, who use the raw features from the 10-K to compute their own models, comparisons, and thus decisions based on the 10-K reports. 

From each 10-K there are several variables that are reported, with Generally Accepted Accounting Principles (GAAP) that dictate how companies should allocate expenses and revenues, both cash and debts, into the 10-K so that investors, banks, analysts, and others have a uniform and repeatable interface to judge the financial health of a company and compare it with competing firms. 
All variables involved in this work are provided in 
\autoref{tbl:variables} 
with industry abbreviations that will be used. 

\begin{table}[!h]
\caption{The set of all common accounting variables extracted from a 10-K that are used in this work. A complete understanding of all variables is not necessary and is provided as a reference so that the common accounting abbreviation can be expanded to its full definition. } \label{tbl:variables}
\centering
\begin{adjustbox}{max width=\columnwidth}
\begin{tabular}{@{}ll@{}}
\toprule
\multicolumn{1}{c}{Abreviation} & \multicolumn{1}{c}{Variable}                 \\ \midrule
AT                              & Total Assets                                 \\
ACT                             & Current Assets - Total                       \\
RECT                            & Net Receivables                              \\
INVT                            & Inventory                                    \\
PPENT                           & Net Property, Plant, \& Equipment            \\
PPEGT                           & Gross Total Property, Plant, \& Equipment    \\
LT                              & Liabilities - Total                          \\ 
LCT                             & Current Liabilities                          \\
DLTT                            & Long-term debt - Total                       \\
NI                              & Net Income                                   \\
SALE                            & Net Sales                                    \\
COGS                            & Cost of Goods Sold                           \\
DP                              & Depreciation                                 \\
AM                              & Amortization                                 \\
XSGA                            & Selling, General and Administrative expenses \\
OANCF                           & Operating Activities - Net Cash              \\
DVP                             & Dividends, Preferred                         \\
XSTF                            & Staff Expenses - Total                       \\
CSHO                            & Common Shares Oustanding                     \\
XAGT                            & Administrative and General Expenses - Total  \\
XEQO                            & Equipment and occupancy Expense              \\
XOPR                            & Operating Expense                            \\
\bottomrule
\end{tabular}
\end{adjustbox}
\end{table}

As the name \textit{balance} sheet implies, it is important that a number of the variables balance out to the same total value. Thus, one can not arbitrarily set every variable in a 10-K, because the numbers would not add up to the same total value. For the purposes of this work, a complete list of all variables that must balance, and how they interact, is given in \autoref{fig:variable_dependencies}. 

\begin{figure}[!h]
\begin{minipage}[t]{0.4\columnwidth}\vspace{0pt}%
\begin{tikzpicture}
[
    every node/.append style = {draw, anchor = west},
    grow via three points={one child at (0.5,-0.8) and two children at (0.5,-0.8) and (0.5,-1.6)},
    edge from parent path={(\tikzparentnode\tikzparentanchor) |- (\tikzchildnode\tikzchildanchor)}]
\node {LT}
    child { node {LCT}
    }
    child { node {DLTT}};
\end{tikzpicture}
\begin{tikzpicture}
[
    every node/.append style = {draw, anchor = west},
    grow via three points={one child at (0.5,-0.8) and two children at (0.5,-0.8) and (0.5,-1.6)},
    edge from parent path={(\tikzparentnode\tikzparentanchor) |- (\tikzchildnode\tikzchildanchor)}]
\node {AT}
    child { node {ACT}
        child {node {RECT} edge from parent [solid] } 
        child {node {INVT} edge from parent [solid] } 
    }
    child [missing] {}
    child [missing] {}
    child { node {PPENT}
        child {node {PPEGT} edge from parent [solid] } 
    };
\end{tikzpicture}
\end{minipage}
\begin{minipage}[t]{0.55\columnwidth}\vspace{0pt}%
\begin{tikzpicture}
[
    every node/.append style = {draw, anchor = west},
    grow via three points={one child at (0.5,-0.8) and two children at (0.5,-0.8) and (0.5,-1.6)},
    edge from parent path={(\tikzparentnode\tikzparentanchor) |- (\tikzchildnode\tikzchildanchor)}]
 
\node {NI}
    child { node {REVT}
        child {node {SALE} edge from parent [solid] } 
        child {node {OPRO} edge from parent [solid] } 
    }
    child [missing] {}
    child [missing] {}
    child { node {PI}
        child {node[dashed] {XOPR} 
            child {node[solid] {COGS} edge from parent [solid]  }
            child {node[solid] {XSGA} edge from parent [solid] }
        edge from parent [dashed]} 
    }
    child [missing] {}
    child { node {PVO} edge from parent [dashed]}
    child { node {AM} edge from parent [dashed]}
    child { node {DP}
        edge from parent [dashed]
    }
    child { node {XAGT}
        child {node[solid] {XEQO} edge from parent [solid] } 
        edge from parent [dashed]
    }
    ;

\end{tikzpicture}
\end{minipage}
    \caption{Solid line indicates a value-added into the parent value, dashed lines a value subtracted from the parent value. This is not a complete diagram, and some of the leaf nodes may actually have more children. For this work, we ignore the children when it is an unnecessary detail.  }
    \label{fig:variable_dependencies}
\end{figure}

Every leaf node of the tree in \autoref{fig:variable_dependencies} we shall term an \textit{atom}, in that it can be directly manipulated. However, manipulating any atom will alter all remaining variables further up in the hierarchy. For example, increasing the DLTT will also increase the LT by the same amount, because DLTT is a descendant of LT. 

A company self-reports all values in the 10-K to the appropriate regulators, and so the 10-K variables are the raw features that investors use to make decisions. This is commonly done by creating \textit{ratios} constructed from the various variables. For example, the \textit{current ratio} is defined as the current assets (AT) divided by current liabilities (LT). While non-linearities can be used in computing a ``ratio'', it is common for all terms to be addition, subtraction, multiplication, and division ~\citep{piotroski_value_2000}. 

The use of these ratios constitutes effective feature engineering on the part of the investors. Linear and logistic regressions, or the ratio itself, are used by investors/regulators as they are interpretable and common machine learning has not produced sufficiently superior results~\citep{hoglund_detecting_2012,sanad_machine_2021}. Thus, there is an actor (the firm) that has a method to alter the features (financial variables) used by other parties (the investors) with significant motivation.%

\subsection{Related Work}

Since the seminal work of \citep{healy_effect_1985} in 1985, fraud models in accounting have been predominantly based around linear models on manually engineered features of accounting variables, termed accounting ratios in practice. 
The vast majority of fraud models in use are thus followers of this overall style of linear models
trained on a small set of known fraud cases against a similarly sized reference population and then applied to future years/companies~\citep{jones_earnings_1991,defond_debt_1994,dechow_detecting_1995,dechow_quality_2002,hribar_errors_2002,rosner_earnings_2003,spathis_detecting_2002}.
The primary fraud model we will study is the M-score~\citep{beneish_detecting_1997}. Despite its age, it is one of the most widely used fraud models. It has been used reliably over decades for financial modeling/portfolio management ~\citep{beneish_predictable_2007,beneish_fraud_2012}, detected real-world incidents of fraud~\citep{noauthor_cornell_1998,rami_2017}, and is still actively modified/used as a benchmark in current accounting literature~\citep{narsa_fraud_2023,lu_research_2020}. As such, we use it as our target model to evade as a major representative of the models in use and one of the most prevalent types of models, if not individual models, used in practice. While linear models have been widely used in machine learning for high-dimensional problems and cases where privacy/provable guarantees are needed~\cite{10516654,wu2024stabilizing,Lu2022,10.1145/3714393.3726507,NEURIPS2023_72235260,NEURIPS2023_72235260,10.1145/3605764.3623910,Ng2004,Zou2005,Fan2008,Baracaldo:2017:MPA:3128572.3140450,Liu:2017:RLR:3128572.3140447}, their use in accounting is unique and interesting in the dependence on domain expert feature engineering that is not as prevalent in machine-learning use cases. Similarly, Bayesian probabilistic programming is often used to encode domain knowledge into the modeling of the algorithm, rather than the feature stage~\cite{Klein24,rubin_bayesianly_1984,Gelman2013,gelman_prior_2006}. The results we are are relatively robust compared to current computer vision~\cite{Biggio2017,Doldo_2025_CVPR,cina_sigma-zero_2024,floris_improving_2023,bryniarski_evading_2021,Rahnama2020,Carlini2017} and natural language processing~\cite{NEURIPS2024_70702e8c,das2025humanreadableadversarialpromptsinvestigation,Carlini2021,Carlini2018} susceptibility to adversarial attacks, despite not being developed for such attacks, is a strong indicator of the intrinsic value in accounting's focus on such feature engineering via accounting ratios and the importance of considering real-world or ``problem space'' attacks~\cite{10136152,raff2023dontneedrobustmachine,Pierazzi2020}.

From an AML perspective, our work deals with regression and multiple objectives. The use of AML in regression has received minimal study, as noted in~\citep{Nguyen_Raff_2019,Gupta_Pesquet_Popescu_Kaakai_Pesquet_Malliaros_Paris_Saclay_2021}. This poses issues when the gradient can change by orders of magnitude because the response changes by orders of magnitudes~\citep{10226611}, but is uniquely different in that the magnitude is an artifact of the information of the domain rather than an ``obfuscated gradient'' that confounds many defensive works in classification problems~\citep{Athalye2018}. 

In addition, little work considers multiple attack objectives simultaneously. Prior works either use a simple weighted average of loss terms~\citep{pmlr-v97-qin19a} and are designed for related/correlated objectives~\citep{williams2023black,bui2023generating,Yang_Xu_Zhang_Hartley_Tu_2024}, where our work has anti-correlated objectives that must be satisfied. 

\section{Threat Model} \label{sec:threat_model}

Following ~\citep{Biggio2014} we will now specify the threat model of our study. First is the high-level inputs, followed by the attacker's optimization targets. This includes a novel formulation of the adversarial regression problem to optimize two variables that must both be attacked successfully for the attack to be viable in practice. Then we will review the strategies that can be used to perturb the variables from the 10-K in a realistic manner. 

For the discussion, we will be using ratios that compare the current year to a previous year (denoted by a  subscript of ${}_{t-1}$). There will thus be $T$ different years of 10-K values and $T-1$ computed scores (excluding the first year). Each company has a different value of $T$ based on data availability. Three matrices are used to encode the data and attack. 
$X \in \mathbb{R}^{T,D}$ is the history of $T$ different financial reports of $D$ different variables. 
$P \in \mathbb{R}^{T,L}$  is the perturbation amounts applied to each of the $T$ different financial reports over time. There are $L$ options for each point in time that indicate which of $L$ different perturbation strategies will be used.  
$M \in \mathbb{R}^{L,D}$ is the sparse matrix of $L$ perturbation strategies mapped to the $D$ different variables that may be impacted.

\subsection{Attack Optimization Targets} \label{sec:threat_model_targets}

Here we detail what the target function to attack, $f(\cdot)$ is. Broadly, many possible targets could be the function of an adversarial attack with the goal of performing fraud. In financial literature, these may often be found by referring to \textit{earnings management} (EM). 
EM is not synonymous with fraud but covers a broad spectrum of financial decisions and accounting decisions that may be made with an eye toward outcomes more appealing to investors and other parties (e.g., a bank determining loan risk/rates). It is worth noting thought that EM can become fraud if pushed too far, and the scope of all possible EM that could be performed using adversarial ML is not the scope of this article.

Our focus is on the matter of two representative targets that a bad actor may target in performing financial fraud. We will detail each individually, which can be represented by \autoref{eq:target}. As we will demonstrate later in \autoref{sec:results}, the straightforward attack of either objective does not result in a holistic strategy that an adversary is likely to employ. For this reason, we will develop a novel strategy to combine both numeric objects into a single optimization target. 

\begin{equation} \label{eq:target}
    \argmin_P f(X + X \odot P M) \quad s.t. \quad |P| \leq \epsilon
\end{equation}

\subsubsection{Financial Target} The first, and intuitive target, is the earnings per share (EPS) of the company, defined by as the Net Income (NI) minus cash paid out to investors in the form of preferred dividends (DVP), divided by the number of common shares outstanding (CSHO). 
\begin{equation}\label{eq:eps}
    \text{Earnings Per Share} (\mathit{EPS}) = \frac{NI-DVP}{CSHO}
\end{equation}
Generally, the higher the EPS, the more attractive the stock looks to investors. This has multiple macro-incentives at the institution level: being able to raise more capital via the sale of shares, more leverage for the purchase of other companies, and more favorable loan terms from banks. On an individual level the potential bad actor, any actor who has meaningful compensation in the form of stock grants will stand to directly benefit from a higher EPS calculation. 

\subsubsection{Evasion Target: M-score} The high incentives to better earnings reports 
by means of
manual human effort in creating fraudulent 10-Ks has been a long-standing issue, and it can take years for fraud to be detected (if at all). For this reason, many have developed models for attempting to detect the risk of fraud or earnings 
manipulation.
We select the M-Score ~\citep{beneish_detecting_1997,beneish_detection_1999} as a representative model to try and evade because the M-score has been in use for several decades~\citep{beneish_predictable_2007,beneish_identifying_2009,beneish_fraud_2012} and was even used by business students to identify Enron's collapse a year in advance ~\citep{noauthor_cornell_1998}. 
The model calculates an M-score based on
a number of known and custom financial ratios (i.e., manual feature engineering) computed from the 10-K statements on a year-over-year basis.
The M-score indicates the degree of manipulation in earnings by a company. For example, a score of $-$2.50 suggests a low likelihood of manipulation, whereas a score exceeding $-$1.78 suggests that the company is likely to be a manipulator. 
Like most models developed in the financial literature~\citep{piotroski_value_2000}, it is a linear model over the ratios. In this particular case, there are eight ratios with corresponding covariates (and bias term) defined as \autoref{eq:m-score}.

\begin{equation} \label{eq:m-score}
\begin{aligned}
    \text{M-Score} =&  -4.84 + 0.92 \cdot \mathit{DSRI} + 0.528 \cdot  \mathit{GMI} + \\
                    & 0.404\cdot \mathit{AQI} + 0.892\cdot \mathit{SGI} + 0.115\cdot \mathit{DEPI} \\
                    & - 0.172\cdot \mathit{SGAI} + 4.679\cdot \mathit{TATA} - 0.327\cdot \mathit{LVGI}
\end{aligned}
\end{equation}

The eight ratios of the M-score are defined 
below. 
Due to space limitations, we will not review each individually. 

In each equation, $t-1$ is used to represent the previous year's value of the variable, and no subscript is used for the current year. 

$$\text { Days Sales In Receivables Index } (\mathit{DSRI}) =\frac{\frac{\mathit{RECT}}{\mathit{SALE}}}{\frac{\mathit{RECT}_{t-1}}{\mathit{SALE}_{t-1}}}$$

$$
\text { Gross Margin Index } (\mathit{GMI}) =\frac{ \frac{\mathit{SALE}_{t-1}-\mathit{COGS}_{t-1}}{\mathit{SALE}_{t-1}}  }{ \frac{\mathit{SALE}-\mathit{COGS}}{\mathit{SALE}} }
$$

$$
\text { Asset Quality Index } (\mathit{AQI}) =\frac{1-\left(\frac{\mathit{ACT}+ \mathit{PPENT}}{\mathit{AT}}\right)}{1-\left(\frac{\mathit{ACT}_{t-1} + \mathit{PPENT}_{t-1}}{\mathit{AT}_{t-1}}\right)}
$$

$$
\text { Sales Growth Index } (\mathit{SGI}) =\frac{\mathit{SALE}}{\mathit{SALE}_{t-1}}
$$

$$
\text { Depreciation Index } (\mathit{DEPI}) =\frac{\frac{\mathit{DP}_{t-1} + \mathit{AM}_{t-1}}{\mathit{DP}_{t-1} + \mathit{AM}_{t-1} + \mathit{PPENT}_{t-1}}}{\frac{\mathit{DP + AM }}{\mathit{ DP + AM+PPENT }}}
$$

$$
\text { SGA Index } (\mathit{SGAI}) =\frac{\frac{\mathit{XSGA}}{\mathit{SALE}}}{\frac{\mathit{XSGA}_{t-1}}{\mathit{SALE}_{t-1}}}
$$

$$
\text { Leverage Index } (\mathit{LVGI}) =\frac{\frac{\mathit{DLTT+LCT}}{\mathit{AT}}}{\frac{\mathit{DLTT}_{t-1} + \mathit{LCT}_{t-1}}{\mathit{AT}_{t-1}}}
$$

$$
\text { Total Accruals to Total Assets } (\mathit{TATA}) =\frac{\mathit{NI}- \mathit{ONCAF}}{\mathit{AT}}
$$

Note that there are both basic variables that are sub-dividable into further components (e.g., the Net Income (NI) used in TATA is constructed from most other variables), and other variables that are final atoms that can not be subdivided. For this reason, a naive strategy of attempting to minimize/maximize each individual ratio may not be effective because it would alter other variables in the hierarchy. Second, most of the ratios are relative year-over-year comparisons, and so greedily altering one year may impact a previous year's calculation. This is a critical consideration in the performance of earnings management because it may create the appearance of volatility in the company's performance, which is itself a negative signal used by analysts~\citep{simko2020financial}.

\subsubsection{Secondary Fraud Score}

The Earnings Per Share (EPS) and Benish's M-Score represent two anti-correlated targets to be minimized simultaneously. Though these are the primary targets of analysis, as EPS is a fundamental business metric and the M-Score is one of the most widely used fraud models, we will include a second fraud detection model as a third objective to be optimized in extended testing. This will allow us to demonstrate that our MVMO algorithm can succeed with more targets simultaneously.

In particular, we will use another well-used accounting fraud model proposed by  Charalambos T. Spathis \citep{spathis_detecting_2002}, which we will term as the \textit{S-Score} going forward. The S-Score model was selected as the coefficients are publicly available, and its relative contrast with the M-score is that it has only three covariates in the model specification and thus far has fewer accounting ratios compared to many alternative models. It's compact specification is given 
by 
$\text{S} =  1.250 + 2.252 \cdot \mathit{INVT}/\mathit{SALE} 
                     -33.029 \cdot  \mathit{NI}/\mathit{AT} 
                     - 6.878\cdot \mathit{WC}/\mathit{AT}$.

One could hypothesize our success against the M-Score is achieved only due to a large number of terms and thus complex interactions between them, as each accounting variable in \autoref{eq:m-score} is composed of the sum of many atoms and thus creates more of the non-linearities that are exploited in neural networks for evasive purposes. By showing that the S-Score also succumbs to our MVMO algorithms, we show a three-way joint optimization and that simpler fraud models are equally at risk. 

\section{Maximum Violated  Multi-Objective Attack} \label{sec:mvmo}
Notably, all of the targets of our adversarial attack are real-valued functions. Most literature on AML is performed on classification tasks, with relatively little work done on regression problems~\citep{Nguyen_Raff_2019,Gupta_Pesquet_Popescu_Kaakai_Pesquet_Malliaros_Paris_Saclay_2021}. In our case, the range of EPS and M-score calculations differ by up to 1000$\times$, making the scale matching to optimize both together difficult. Also problematic is that the EPS and M-score both range from $-\infty$ to $\infty$, and in all cases, we wish to maximize EPS and minimize the M-score.

Our results will show that simply optimizing for the average of multiple least-squares objectives will not produce a satisfying outcome. Thus, our strategy will focus on the nature of the threat model: \textit{all objectives must be satisfied for the result to satisfice}. 
If EPS increases (good) but the M-score also increases (bad), that means the fraud is more likely to be detected and so undesirable. Similarly, if M-score decreases but the EPS decreases, then there is a low desire to perpetrate the fraud due to negative outcomes. 
In our threat model, there is no amount by which EPS may increase that will make up for an increase in the M-score. For this reason we will focus on a gating strategy that will depress any objective that is satisfied in favor of other objectives that are unsatisfied. 

Second, since we have objectives of infinite support, there is the real-world consideration that they may have different sensitivities. That is to say, it may be easier to change one objective by a large magnitude than another objective. To balance this, we will instead focus on the magnitude of the changes made by taking the transformation $g(\Delta) = \operatorname{sign}(\Delta) \log(|\Delta|+1)$, where $\Delta$ is the change in the objective compared to its original value. This transformation will maintain the sign of the change that has occurred (hence the +1 so that the origin remains the same pre/post-transformation) but only increase logarithmically. In this way, a more sensitive function will not gain undue favor in the numerical optimization. 

This leads to our numeric optimization (shown for two variables for illustration purposes) in \autoref{eq:target_joint}, where we compute the transformation of EPS and M-score (M)  as $\alpha$ and $\beta$, respectively. A larger EPS is better, so the minimization goal is $-\alpha$. A smaller M-score is better, so $\beta$ can be optimized directly. This is converted to a final optimization score by taking the weighted average of $-\alpha$ and $\beta$ via the softmax function. 

\begin{equation} \label{eq:target_joint}
\begin{aligned}
\alpha = g(\mathit{EPS}(X) - \mathit{EPS}(X + X \odot P M)) \\
\beta = g(\mathit{M}(X) - \mathit{M}(X + X \odot P M)) \\
\argmin_P  [-\alpha, \beta]^\top \operatorname{Softmax}([-\alpha, \beta] \cdot C) \quad s.t. \quad |P| \leq \epsilon
\end{aligned}
\end{equation}

The softmax is computed so that more weight is placed on the metric that is currently performing worse by our goals. e.g., if $\alpha = 100$, and $\beta= 5$, we have already succeeded in making the EPS significantly higher, but have done so at increased risk of detection. So the $\operatorname{Softmax}([-100, 5]) \approx [0, 1]$ will result in $-100 \cdot 0 + 5 \cdot 1 = 5$, so that the optimization effort is spent on reducing the M-score. 

\begin{algorithm}
\caption{Maximum Violated  Multi-Objective Attack} \label{algo:mvmo}
\begin{algorithmic}
\REQUIRE Feature set $x$, with $K$ victim functions $f_{1}(\cdot), \ldots, f_{K}(\cdot)$ to minimize. Maximum number of iterations $T$, and constraint projection $\pi(\cdot)$. Exaggeration constant $C$ (initialized to 1 in our tests). 
\STATE $x' \gets x$ 
\STATE $\mathit{d}_0 \gets [f_{1}(x'), f_{2}(x'), \ldots, f_{K}(x')] $ \COMMENT{\textcolor{blue}{Reference objectives needed to ensure we do not over-optimize one at the expense of others.}}
\FOR{$t = 1$ \textbf{to} $T$ }
    \STATE $\mathit{d}_t \gets [f_{1}(x'), f_{2}(x'), \ldots, f_{K}(x')]$ \COMMENT{\textcolor{blue}{$d$ holds the directions that each value takes to make both a minimization goal.}}
    \STATE $\mathit{\chi} \gets \operatorname{sign}(\mathit{d}_t-\mathit{d}_0) \odot \log\left(|\mathit{d}_t-\mathit{d}_0|+1\right)$  \COMMENT{\textcolor{blue}{Transform the targets to be insensitive to order-of-magnitude differences.}}
    \STATE $\mathcal{L} \gets \mathit{\chi}^\top \operatorname{Softmax}(\mathit{\chi} \cdot C)$ \COMMENT{\textcolor{blue}{Final optimization gibes the least-improved the most weight.}}
    \STATE $x' \gets  \pi(x' + \nabla_{x'} \mathcal{L})$
\ENDFOR
\STATE \textbf{return} $\delta = x-x'$ \COMMENT{\textcolor{blue}{Adversarial perturbation found}}
\end{algorithmic}
\end{algorithm}

This forms the inspiration and method of our Maximum Violated  Multi-Objective (MVMO), which focuses optimization on only the objectives that are currently unsatisfied relative to the others, as shown in \autoref{algo:mvmo}. The projection function $\pi$ keeps the solution constrained to the feasible attack space, and significantly improves upon naive PGD attacks. Note that in the case of perfectly correlated victim functions where $f_{i}(x) = f_{j}(x)$, our method will reduce to standard PGD as the softmax will compute the simple average in such cases. Our approach is simple to implement, and as we will show, significantly more effective in practice.

\subsubsection{Attacker Constraint}

We have now enumerated the goals of the attacker. The constraints are simply to specify a limit on $\epsilon$, which is interpretable as the maximum relative change allowed to any atom in the original 10-K report. 
A relative change in a variable by over 70\% occurs in real-life fraud~\citep{rami_2017}. For this reason, we will consider a maximum limit of $\epsilon \leq 40\%$ so that our results are well within the scope of perturbations committed in real-world situations. 

Note that some atom/leaf variables in our model, like DLTT, actually have more children in the complete specification of the financial variables. By treating DLLT and others as an atom instead, \textit{we effectively reduce the total strength of a true theoretical attack}. Consider that DLTT may be 0, when in actuality it has two large values $Z$ and $-Z$ that contribute to its calculation. In this case, the net score is 0, and any percent change will be zero, so no valid perturbation is possible in our threat model. In actuality, one of the two sub-components could have been modified, thus increasing the total set of valid attacks. Our simplification is done as a matter of exposition (we do not wish to teach the reader all of the accounting) and practicality (the sub-accounts are not always used by all companies, making the code more complex). 

We note that it may seem uncessary for an iterative gradient attack due to the linear nature of accounting variables (addition, subtraction, multiplication, and division). In fact, to keep the attack feasible, it is key to constraint the year-to-year perturbatinos to be consistent, which makes the problem non-convex.
\begin{theorem}
    Given a function $f(x) = \sum_{t=1}^{T} \alpha_t \frac{\beta_{t-1}}{\beta_{t}}$ where $\beta_t$ is the $t$'th year's accounting variable, making $\beta_{t-1}/\beta_t$ the accounting ratio, and $\alpha_t$ the weight of the variable. The objective $f(x)$ is non-context over the sum of years $T$.
\end{theorem}

\begin{proof}
    The diagonal of the hessian is $\frac{\partial^2 }{\partial \beta_i^2} \left( \sum_{t=1}^{T} \alpha_t \frac{\beta_{t-1}}{\beta_{t}} \right) = \frac{2 \beta_{i-1} \alpha_{t-1}}{\beta_i^3}$, for which $\frac{2 \beta_{i-1} \alpha_{t-1}}{\beta_i^3} < 0$ whenever $\beta_i < 0$. We thus fail to satisfy the necessary condition that $f''(x) > 0 \forall x$ for $f(x)$ to be a convex function.
\end{proof}

\subsection{Perturbation Strategies} \label{sec:threat_model_preturbation}

We will detail the strategies that we encode into the threat model below. In each case, we have two corresponding variables that are added to the matrix $M$, with each row representing one pair of variables mentioned below. For example, if the relationship of the $i$'th variable is an anti-correlated impact on the $j$'th variable, the relationship will be encoded as $M[r,i] = 1$ and $M[r,j] = -1$, where $r$ is an arbitrary row that is holding the current relationship. Thus when the matrix product is taken with $P$, the corresponding $i,j$ variables will be altered in a one-to-one relationship. 

By the nature of how $P$ and $M$ are represented, the strategies used in this section are single-year strategies, actions that can be taken in one year to influence the final reported 10-K, and thus the ratios that will be used to evaluate the company. This is a realistic threat model in that a bad actor could apply these techniques to optimize for the current year's 10-K. While more complex multi-year strategies are possible, they become more dependent on the degree of long-term financial planning conducted at each firm, which can be industry and Chief Financial Officer (CFO) specific. 

To ensure that our strategies are plausible to a real-world threat model, they were constructed with guidance and collaboration with domain experts: including processional certified public accountants (CPAs) and lawyers. It is critical that we emphasize that these are not legal for someone to perform, but reflect strategies observed in real-world financial fraud. 

\subsubsection{Cost of Goods Sold (COGS) based Strategies} COGS is in general expenses incurred that are directly traceable back to the production of a product (e.g., physical materials the product is made from). It, along with Selling, General, \& Administrative (XSGA) expenses can have significant overlap in the nature of their accounting: both at a high level are tracking expenses related to the production of revenue, the difference is where the expense is traceable to. However, this similarity means an expense can be easily moved between the two accounts. Both COGS and XSGA are used in the Beneish-M score, impacting different variables, and so incentivizes allocation of the expenses to the minimally penalized variable. For the same logical reasons, a third account for Staff Expenses (XSTF) can also be used to reallocate funds, while the XSTF is less frequently used in standard financial ratios. Thus for a \$1 increase in COGS, the adversary can alter XSGA or XSTF by -\$1 in this their model (and in the reverse as well). 

A second variable that is not obviously connected to COGS is the Inventory (INVT). This is because, on a firm's balance sheet, INVT and COGS are expected to balance via a third factor as:
$\mathit{COGS} = \text{Beginning Inventory} + \text{ Purchases } - \text{Ending Inventory}$.

Where the inventory sold during the year produces the revenue (and corresponding tracked expense) that goes into the calculation of COGS. However, tracking inventory is notoriously error-prone. This thus makes it a convenient place for an adversary to intentionally ``poorly track'' their inventory to inflate/deflate it as desired, and thus alter the COGS. Thus, a \$1 increase in COGS can also be achieved by a -\$1 change to INVT, and vice versa. 

\subsubsection{Debt Term Alteration} Debt is categorized into Current Liabilities (LCT) that are due within the year, and Lon-term debt (DLTT) that is due in more than a year's time. DLTT becomes LCT as the debt nears its due date. For a bad actor, so long as they pay all debts due, they can miss-report DLTT as LCT by claiming in the following year that the LCT was paid off, but new DLTTs were acquired after the year. Similarly, LCT can be reported as DLTTs by that are explained away the following year as being paid off early. In each case, the debtors are paid on time, but the accounting is done fraudulently. Thus DLTT and LCT have a \$1/-\$1 invertible relationship. 

\subsubsection{Phantom Sales} A firm's revenue (SALE) when reported on a 10-K does not reflect \textit{just} literal cash received in the fiscal year. Obligations from other parties to pay the firm in the near future, but that have not yet been actually paid, also contribute to SALE. This is accounted for with the Accounts Receivable (RECV), and so fraudulent SALE volume can be achieved by adding a corresponding amount of RECV.  The strategy thus has a \$1/\$1 relationship in $M$. This could be explained away the following year by claiming that the purchaser fell through and failed to make payment. 

\subsubsection{Hiding Property Expenses} The previous strategies mentioned are all contained within the balance sheet of a 10-K, which is a reflection of the company's net value from inception accumulated to the current fiscal year. To alter the income sheet of a 10-K, we can use the property-related reports. Equipment and Occupancy Expense (XEQO, on the income statement) that accounts for maintenance of Property, Plant, and Equipment (PPEGT, on the balance sheet) provides a direct connection. Altering XEQO by -\$1 (and thus, increasing net income) by claiming it instead as the purchase of new PPEGT (and thus increasing it by \$1) creates a double-dipping effect on the total finances. 

\subsubsection{Depreciation Tax Shield} Depreciation (DP) in accounting is the accumulation of ``negative'' value to an asset already purchased. This is valuable financially because DP has a ``tax shield'' effect, where DP is removed from the total income for tax calculation purposes, even though DP does not correspond with a cash transfer in the current year. DP calculations are highly subjective, alterable, and can be challenging to track. A company can fraudulently claim accelerated or non-existent DP to obtain an undue tax credit. This is encoded as just a single $M[r,DP] = 1$ in the perturbation strategy matrix.

\section{Results} \label{sec:results}

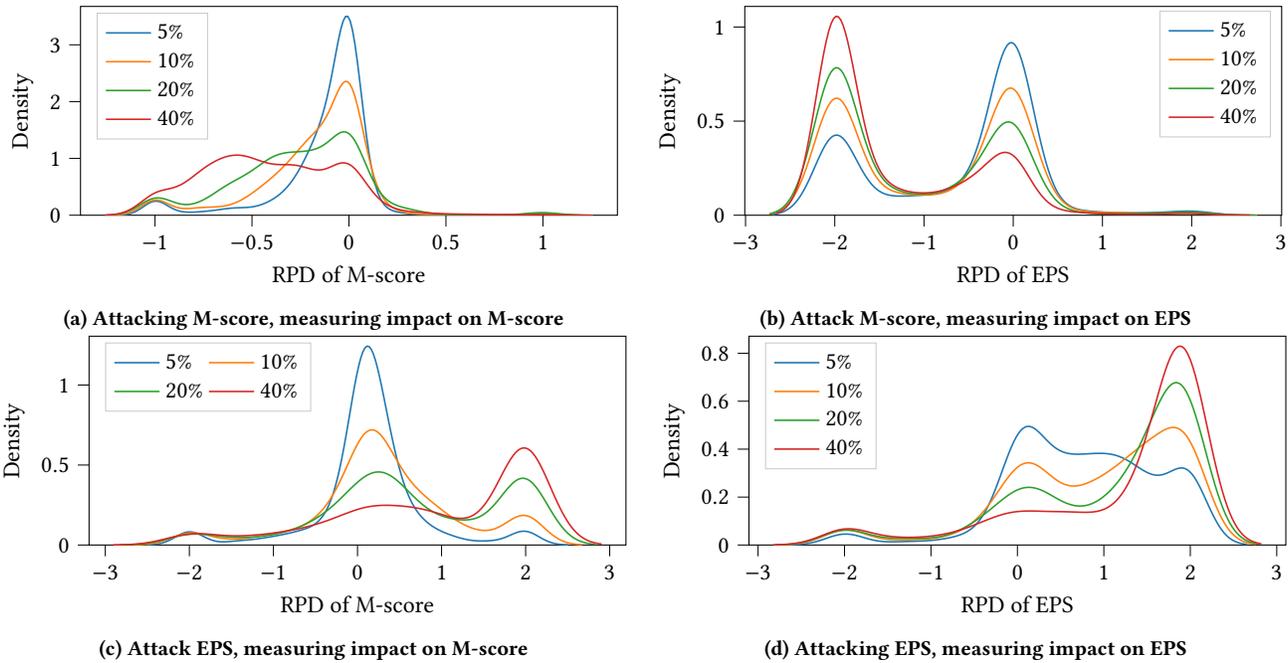
\begin{figure*}
    \begin{subfigure}{0.49\textwidth}
        \centering
        \input{figures/attackM_rpd_M}
        \caption{Attacking M-score, measuring impact on M-score}
    \end{subfigure}
    \begin{subfigure}{0.49\textwidth}
        \centering
        \input{figures/attackM_rpd_eps}
        \caption{Attack M-score, measuring impact on EPS}
    \end{subfigure}
    \begin{subfigure}{0.49\textwidth}
        \centering
        \input{figures/attackEPS_rpd_M}
        \caption{Attack EPS, measuring impact on M-score}
    \end{subfigure}
    \begin{subfigure}{0.49\textwidth}
        \centering
        \input{figures/attackEPS_rpd_eps}
        \caption{Attacking EPS, measuring impact on EPS}
    \end{subfigure}
    \caption{Kernel Density plot results when using the simple strategies of attacking just the M-score (top row) or just EPS (bottom row) for differing values of $\epsilon$. The x-axis in each case is the RPD (1 = 100\%) showing the effect on the dataset. The EPS goes down when targeting M-score (top-right), which defeats the purpose of the adversary committing the fraud, making it non-viable. Attacking EPS is easily detected by raising the M-score (bottom-left). }
    \label{fig:simple_target_results}
\end{figure*}

To produce our results, we implement our model in the JAX~\citep{jax2018github} framework. JAXOPT~\citep{jaxopt_implicit_diff} is used to perform the adversarial search using projected gradient descent (PGD)~\citep{Madry2018} in order to respect the $\epsilon$ limits. We will evaluate values of $\epsilon \in \{5\%, 10\%, 20\%, 40\%\}$. Each variable is encoded with respect to \autoref{fig:variable_dependencies} such that any perturbation to one atom/leaf node automatically propagates to all other calculations. This allows implementation via automatic differentiation through the financial ratios. Our data is sourced from ~\citep{10.1111/1475-679X.12292}, which contains 402 companies that have confirmed years of actual financial fraud occurring. There are a cumulative 979 years of 10-Ks in the data to be altered. Each company has an average of 7.9 years of available data. We obtain the original 10-K data in a standard format via the Wharton Research Data Services ~\citep{WRDS} database. This gives us ideal data to evaluate if our attacks would reduce the M-score commonly used for the identification of companies that are at risk of earnings management. 

In each experiment, we will examine the percentage of years that an attacker could apply the MVO (or alternative) algorithm and satisfy the attacker's goal under the threat model: reducing M-Score and increasing EPS simultaneously.
Because sign changes occur, we use the Relative Percent Difference 
$RPD(a,b) = \frac{a-b}{(|a|+|b|)/2}$,
to compute the relative change in EPS and M-score, which would be problematic for the standard relative change calculation of $(a-b)/a$. 

As baselines, we will consider the SA-MOO multi-object approach of ~\cite{williams2023black}, which we find fails due to sensitivity to order-of-magnitude change in the objectives and over-focusses on the easier metric of EPS because it allows driving the cumulative reward up by orders of magnitude. Second, we consider the Manifold ELBO of ~\cite{Yang_Xu_Zhang_Hartley_Tu_2024}, which we find fails due to sampling noise in its stochastic estimation of the ELBO. Finally, we test three forms of PGD: attacking the M-score, EPS, and the average of M-score and EPS. In each case, PGD alone is insufficient. 

\begin{table}[!h]
\caption{Based on the attack method and $\epsilon$ (1st \& 2nd column), the percent that satisfy both (3rd) a lower M-score and higher mean EPS. The MVMO method success rate goes down slightly as $\epsilon$ increases, but the size of the average attack's influence increases (4th and 5th columns). Only the MVMO method of \autoref{eq:target_joint} satisfies a large number of cases.  } \label{tbl:results}
\centering
\begin{adjustbox}{max width=\columnwidth}
\begin{tabular}{@{}lrrrr@{}}
\toprule
\multicolumn{1}{c}{\begin{tabular}[c]{@{}c@{}}Attack\\ Method\end{tabular}} & \multicolumn{1}{c}{$\epsilon$} & \multicolumn{1}{c}{\begin{tabular}[c]{@{}c@{}}Satisfy\\ Goals (\%)\end{tabular}} & \multicolumn{1}{c}{\begin{tabular}[c]{@{}c@{}}Avg. M-score \\ RPD (\%)\end{tabular}} & \multicolumn{1}{c}{\begin{tabular}[c]{@{}c@{}}Avg. EPS \\ RPD(\%)\end{tabular}} \\ \midrule
SA-MOO & 5\% & 6.30    & 13.71 & 53.78 \\
SA-MOO & 10\% & 5.72    & 35.40 & 63.95 \\
SA-MOO & 20\% & 5.17    & 72.89 & 76.29 \\
SA-MOO & 40\% & 4.15    & 104.86  & 86.77 \\
Manifold ELBO & 5\% & 1.61    & 2.47 & 3.19 \\
Manifold ELBO & 10\% & 1.96    & 0.66 & 6.61 \\
Manifold ELBO & 20\% & 0.41    & 9.61 & 12.59 \\
Manifold ELBO & 40\% & 0.77    & 4.20  & 12.67 \\
PGD-M & 5\% & 4.60    & -12.04 & -60.31 \\
PGD-M & 10\% & 3.88    & -18.04 & -88.91 \\
PGD-M & 20\% & 2.66    & -26.43 & -110.42 \\
PGD-M & 40\% & 1.84    & -40.6  & -135.29 \\
PGD-EPS  & 5\% & 13.28   & 12.17  & 74.69 \\
PGD-EPS  & 10\% & 12.05   & 31.43  & 88.76 \\
PGD-EPS  & 20\% & 10.93   & 64.72  & 105.93 \\
PGD-EPS  & 40\% & 8.89    & 93.11  & 120.49 \\
PGD-Avg	&	5\%	&	2.66	&	-5.42	&	-42.93	\\
PGD-Avg	&	10\%	&	3.47	&	-20.92	&	-49.26	\\
PGD-Avg	&	20\%	&	3.88	&	-43.95	&	-58.72	\\
PGD-Avg	&	40\%	&	2.25	&	-55.83	&	-62.5	\\
MVMO & 5\% & \textbf{49.13} & -8.83 & 32.79 \\
MVMO & 10\% & \textbf{61.18} & -11.52 & 43.35 \\
MVMO & 20\% & \textbf{65.99} & -14.01 & 51.53 \\
MVMO & 40\% & \textbf{63.84} & -15.54 & 57.62 \\ 
\bottomrule
\end{tabular}
\end{adjustbox}
\end{table}
The simple attack baseline strategy from \autoref{sec:ablated_attacks}  that numerically evaluates one of either M-score or EPS are shown in 
\autoref{fig:simple_target_results} (i.e., target only one of the two outcomes). As can be seen, the two strategies are naturally correlated. Decreasing the M-score to evade detection also decreases the EPS. Since the purpose of the fraud is to falsely bolster the financial performance, the attack becomes undesirable. Similarly, attacking the EPS also raises the M-score, increasing the risk of fraud being detected. Thus, the naive strategy of attacking one single metric is inadvisable. 

We also tested a number of simple ablations for merging M-score and EPS targets. This includes taking a weighted average, testing increments of 10\% for the weighting, and taking the difference in the mean decrease in scores (measured over all the years of a company's data). In all cases, we observed one of the two scenarios outlined in \autoref{fig:simple_target_results}: M-score and EPS both decreasing or both increasing.

Our novel MVMO attack strategy that works to equalize the order-of-magnitude change in both goals performed significantly better. The results for all methods are shown in \autoref{tbl:results}, where the third column indicates the percentage of years that satisfy the goal of decreasing RPD while also increasing EPS. The MVMO optimization we propose in \autoref{eq:target_joint} is the only strategy that has any meaningful amount of success. PGD applied to just M-score (PGD-M) and EPS (PGD-EPS) each only achieve their target goal, at the cost of the other. Attempting to average the losses (PGD-Avg) performs even worse than naively ignoring a goal. Similarly, normalizing gradients, rescaling, and clipping did not improve performance. Our MVMO is the only mechanism able to satisfy the two anti-correlated targets simultaneously. 

Notably, the \textit{average} relative changes can be large, by up to an absolute 135\%. The EPS changes are large in magnitude in every case, likely because the EPS is one ratio where the M-score is a regression of multiple ratios, and so is more robust. Our MVMO attack does a better job of satisfying both goals by reducing the gap in their absolute magnitudes, which is the intended effect. A key takeaway from these results is that \textit{significant EPS manipulation is possible with no detectable impact on fraudulent activity}. 

That we can reach a  66\% success is also notable in the much higher success rates often seen in the deep-learning literature~\citep{Biggio2017}. This result is more congruent with older adversarial machine learning literature that studied simpler linear and kernel methods with higher intrinsic robustness~\citep{Biggio:2012:PAA:3042573.3042761}. 

A key baseline in \autoref{tbl:results} is PGD-Avg, which takes the average of the EPS and M-Score objective. This had a lower satisfying rate than naively optimizing just the EPS, at 4\% for PGD-Avg compared to 13\% for PGD-EPS. This demonstrates how the scale and individual target sensitivity can lead to unintuitive outcomes when not factored into the attack strategy accordingly. 

The ``price'' to pay for MVMO's success in this context is a reduction in the average RPD of the target metrics. Though this tradeoff is to be expected, the relative cost of the tradeoff is inconsequential to the attacker's perspective. Any increase in EPS at no increased risk is a win for the attacker, and the EPS change is only a factor of two smaller compared to the naive PGD-EPS optimizer. Similarly, the M-Score RPD reduction is approximately $2.5\times$ smaller for each budget of $\epsilon$. 

\begin{figure}[!h]
    \centering
    \input{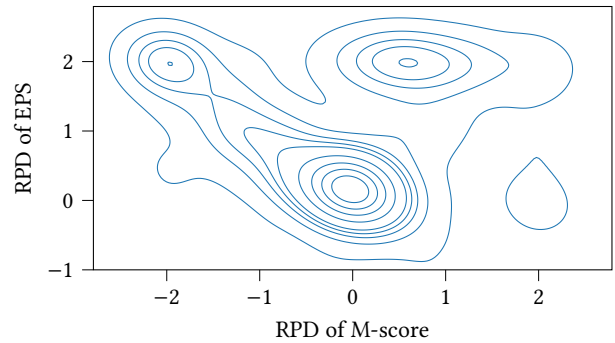}
    \caption{Two-dimensional Kernel Density Estimate (KDE) of the RPD scores for the M-score (x-axis, lower is better) and the EPS (y-axis, higher is better). 49\% of the density exists in the satisficing zone of the top-left quadrant, which are successful attacks satisfying both objectives. The KDE shows that the highest density region is at the center with no relative change to either variable, indicating a large number of points that are difficult to optimize. }
    \label{fig:kde_joint}
\end{figure}

One might also ask about the distribution of outcomes beyond a binary success/failure criterion. For example, if there are many M-Score RPDs that are positive but essentially zero, the risk of attacker success in practice may be greater still. 
This hypothesis is confirmed by a Kernel Density Estimate (KDE) plot of both objectives in \autoref{fig:kde_joint}. A large density near zero confirms the hypothesis and indicates that a large population of companies is challenging to optimize both EPS and M-score concurrently; however, many EPS increases can be obtained with minor M-score increases.  A key takeaway from this initial result is that \textit{financial ratios in use today are reasonably effective against adversarial attacks}. 
This result is largely explained by the theoretical work of \citep{10048547}, who showed that the risk of linear regression models is primarily a function of the number of features in the input. Because most financial models use a relatively small number of ratios, the total size of the attack space is constrained. 66\% is a much lower attack success rate compared to machine learning literature at large. This is in many ways a positive indicator of the tools that have evolved over time in accounting, which makes sense given the adversarial environment that had already formed between duplicitous firms and regulators/the market. The results are far from allowing accounting research or defensive AML research to ignore this application area though. 

At the same time, this amplifies the fact that the high-average impact on EPS is obtained from a second population that routinely obtains up to 200\% increases in EPS. This is a significant effect size that would have a material impact on all agents involved. This poses a non-trivial risk to investors, regulators, and the population at large whose investments may be held with these firms. The magnitude of the attack successes may inform regulators of the scope of risk, and the degree of attack success against different firms may help inform where to most judiciously apply their regulator actions. The results in simulation do not account for additiona actions a bad actor may take beyond simple 10-K reporting fraud if they had the tools available to perform these more advanced manipulation strategies, and could further re-structure their corporate operations to make AML more amenable to future year's reporting. 

\subsection{Ablating Additional Objective}

To evaluate our model when a second fraud model is in consideration, we use the S-Score \citep{spathis_detecting_2002}. EPS and M-Score are anti-correlated objectives, while M-Score and S-Score should be correlated as they are both fraud detection models. This answers two questions: 1) How well does an evasive 10-K report generalize to fraud models not previously seen? 2) Can MVMO successfully optimize three simultaneous objectives? 

\begin{table}[!h]
\caption{
Results of \autoref{algo:mvmo} when optimizing EPS and M-Score (MVMO-M), and optimizing an additional third objective for the M-score (MVMO-MS). Columns 
show the percentage of times that the respective objective was satisfied.
The S-Score is correlated with M-Score as another fraud model.
} \label{tbl:mvmo_three_objectives}
\resizebox{\columnwidth}{!}{%
\begin{tabular}{@{}lrcccc@{}}
\toprule
\multicolumn{1}{c}{Algo} & \multicolumn{1}{c}{$\epsilon$} & EPS\%  & M-score\% & S-Score\% & 3-way Joint\% \\ \midrule
MVMO-M                     & 5\%                           & 76.71 & 70.38    & 53.22    & 27.17        \\
MVMO-M                     & 10\%                          & 80.39 & 78.96    & 54.34    & 34.22        \\
MVMO-M                     & 20\%                          & 84.27 & 80.39    & 56.38    & 37.18        \\
MVMO-M                     & 40\%                          & 84.88 & 77.53    & 55.67    & 34.01        \\
MVMO-MS                    & 5\%                           & 77.63 & 68.64    & 87.64    & \textbf{42.59}        \\
MVMO-MS                    & 10\%                          & 80.18 & 78.24    & 89.58    & \textbf{54.14}        \\
MVMO-MS                    & 20\%                          & 84.27 & 81.31    & 92.34    & \textbf{62.31}        \\
MVMO-MS                    & 40\%                          & 85.70 & 78.86    & 93.46    & \textbf{62.72}        \\ \bottomrule
\end{tabular}%
}
\end{table}

The results of this experiment are shown in \autoref{tbl:mvmo_three_objectives}, where MVMO-M is the same model from \autoref{tbl:results} that optimized just two objectives: EPS and M-Score; MVMO-MS optimizes both M and S-scores.  MVMI-MS has almost equivalent satisfaction rate against three targets as MVMO did to just two, showing more objectives can be satisfied. As can be seen in the final ``3-way Joint'' column, MVMO-M has a non-trivial satisfaction rate of 34\% against all three desiderata, even though it was never previously exposed to the S-Score. This result alone is encouraging from the attacker's perspective in that there is evidence they can evade detection from models that they have not previously considered. We remind the reader that the S-Score is a fraud detection model, so even a 50\% satisfaction rate against just the S-Score would be better than random guessing as the S-Score is supposed to detect fraud, which is actively being performed.

\section{Ablated Attack Objectives} \label{sec:ablated_attacks}

We note in this section prior attempts at implementing an attack, which necessitated our eventual design. First, we note the obvious strategy is to perform mean squared error (MSE) loss on the $-$EPS (i.e., maximize) and M-score (i.e., minimize) objectives. When restricted to just the M-score alone we found MSE based optimization ineffective due to over-optimizing just one year's M-score. This was exacerbated when attempting to add EPS, due to its different scale, and resulted in zero years where both EPS went up and M-score went down. This was repeated with the $L_1$ loss and the same issue was observed. 

Our second attempt to implement the attack was informed by noticing that the M-score for one particular year may become the target of optimization because it is so much lower than all other values, causing the optimization process to over-fit the year. This occurs because of the use of ratios in the M-score calculation, and if one year has a particulary small numerator/denominator in one year, it may become easier to leverage a unit change in the input for an out-sized impact on the average M-score across years. Thus, we used a maximum-violator policy and took only the maximum M-score over all $T$ years in each call to the objective function so that the worst violating year was optimized down. This requires proportionally $T$ more PGD iterations to influence all potential years, however, we found it the second-best strategy in our testing. Our baseline results in \autoref{sec:results} all use this strategy for M-score optimization, and use standard MSE-based loss for EPS optimization. 

Other similar strategies for the M-score were tested as well. This includes clipping M-scores to a minimum value that was iteratively lowered as a function of the larger M-score, using a power function to squash the larger magnitude M-scores into smaller ranges, and top-k M-score selection. All resulted in comparatively equivalent performance in practice, and so we keep the max selection as our initial baseline for ease of implementation. 

\section{Conclusion} \label{sec:conclusion}

In this work, we have demonstrated the applicability of AML to financial fraud modeling. Using a novel multi-goal optimization strategy that is designed to handle anti-correlated objectives, 
we are able to demonstrate significant $\approx50\%$ of years a company can inflate earnings by as much as 200\% while having an average 15\% decrease in fraud risk scores.

In consideration of publishing this work, we note that our approach has no impact on the ability to evade financial auditing, a regular requirement of large companies. While auditing is not flawless, and there are concerns about reduced auditing effectiveness as the market concentrates on the ``big four'' accounting firms\footnote{\url{https://www.ifiar.org/?wpdmdl=15294}}, it remains a secondary check on malicious behavior. Financial ratios are also not the only source by which fraud is detected. The role of the press is non-trivial in detecting financial fraud~\citep{miller_press_2006}. Executives are also expected to engage with investors in annual meetings, which can be revealing of fraud risk~\citep{larcker_detecting_2012}, and non-participation would pose an even greater red-flag signal. Given these other checks on fraudulent behavior and the skills gap for CFOs and accountants to perform our attack, we believe the interest in disclosing the capability to encourage study and remediation outweighs the risk of it being used without eventual detection. 

In making this discovery public, it is worth noting some of the ways in which we hope our results will lead to positive outcomes. First, we note that this interdisciplinary study required collaboration across multiple disciplines to achieve its goals. Encouraging this collaboration, particularly between government, industry, and academia, to study this problem is an immediate goal. In particular, the use of linear models increases the possibility of developing provable bounds on attack risk and success rates, which can lead to better fiscal oversight and the allocation of regulatory resources. In particular, if regulators can focus their time on cases with the greatest risk of successful fraud, then all parties can ultimately benefit.

\bibliographystyle{ACM-Reference-Format}
\bibliography{references,aaai24,moreRefs}

\end{document}

%% file: figures/attackM_rpd_M.tex
\begin{tikzpicture}

\definecolor{crimson2143940}{RGB}{214,39,40}
\definecolor{darkgray176}{RGB}{176,176,176}
\definecolor{darkorange25512714}{RGB}{255,127,14}
\definecolor{forestgreen4416044}{RGB}{44,160,44}
\definecolor{lightgray204}{RGB}{204,204,204}
\definecolor{steelblue31119180}{RGB}{31,119,180}

\begin{axis}[
legend cell align={left},
legend style={
  fill opacity=0.8,
  draw opacity=1,
  text opacity=1,
  at={(0.03,0.97)},
  anchor=north west,
  draw=lightgray204
},
tick align=outside,
tick pos=left,
x grid style={darkgray176},
xlabel={RPD of M-score},
xmin=-1.3842182265283, xmax=1.3842182265283,
xtick style={color=black},
y grid style={darkgray176},
ylabel={Density},
ymin=0, ymax=3.6756650482558,
ytick style={color=black},
width=\columnwidth,
height=0.5\columnwidth
]
\addplot [semithick, steelblue31119180]
table {%
-1.18834521701548 0.00255840827434157
-1.17640204900527 0.00445214610871281
-1.16445888099507 0.00747321126808416
-1.15251571298486 0.012100232712139
-1.14057254497465 0.018899114042474
-1.12862937696445 0.0284751478635098
-1.11668620895424 0.0413891550517146
-1.10474304094404 0.0580399440327198
-1.09279987293383 0.078526627623441
-1.08085670492362 0.102516472095143
-1.06891353691342 0.129152866556097
-1.05697036890321 0.157039117861202
-1.04502720089301 0.184323910569818
-1.0330840328828 0.208893499173465
-1.0211408648726 0.228648331151541
-1.00919769686239 0.241815683380216
-0.997254528852184 0.247233962721395
-0.985311360841978 0.244545172109027
-0.973368192831773 0.234251011003993
-0.961425024821567 0.217620219955439
-0.949481856811361 0.196470411849113
-0.937538688801155 0.172875886008022
-0.925595520790949 0.148865760616852
-0.913652352780744 0.126171736715323
-0.901709184770538 0.106065438797439
-0.889766016760332 0.0892990949123497
-0.877822848750126 0.0761386903741267
-0.86587968073992 0.0664619181661292
-0.853936512729714 0.0598868398530213
-0.841993344719509 0.0559001568535649
-0.830050176709303 0.0539630680018339
-0.818107008699097 0.0535839117928697
-0.806163840688891 0.0543568876039066
-0.794220672678686 0.0559731509183695
-0.78227750466848 0.0582138475595093
-0.770334336658274 0.0609345801462044
-0.758391168648068 0.0640484038617278
-0.746448000637862 0.0675110129999246
-0.734504832627657 0.0713085481324378
-0.722561664617451 0.0754463658784712
-0.710618496607245 0.0799366161140601
-0.698675328597039 0.084783456664957
-0.686732160586833 0.0899666336334408
-0.674788992576628 0.0954261567904397
-0.662845824566422 0.101052118861796
-0.650902656556216 0.106683804653225
-0.63895948854601 0.112120939055916
-0.627016320535804 0.117147454411599
-0.615073152525599 0.121565082198374
-0.603129984515393 0.12523118416585
-0.591186816505187 0.12809337835843
-0.579243648494981 0.130213354358285
-0.567300480484775 0.131774069108689
-0.55535731247457 0.133067967536464
-0.543414144464364 0.13446813114338
-0.531470976454158 0.136388144272763
-0.519527808443952 0.139238847937046
-0.507584640433746 0.143390334912332
-0.495641472423541 0.149145548187654
-0.483698304413335 0.156728423690921
-0.471755136403129 0.166285872393498
-0.459811968392923 0.177900231268092
-0.447868800382717 0.19160785669069
-0.435925632372512 0.207420252641982
-0.423982464362306 0.225345757315556
-0.4120392963521 0.245411236763886
-0.400096128341894 0.267683530264158
-0.388152960331688 0.292289312868297
-0.376209792321483 0.31943018320364
-0.364266624311277 0.349388337670631
-0.352323456301071 0.382518360809806
-0.340380288290865 0.419223029606744
-0.328437120280659 0.459915165446939
-0.316493952270454 0.504972081953017
-0.304550784260248 0.55469224396487
-0.292607616250042 0.609263903640205
-0.280664448239836 0.668752347499259
-0.26872128022963 0.733107071761652
-0.256778112219425 0.802185024229865
-0.244834944209219 0.875783689497343
-0.232891776199013 0.953680032391443
-0.220948608188807 1.03567783788597
-0.209005440178601 1.12167379893517
-0.197062272168396 1.21175662437838
-0.18511910415819 1.3063478206917
-0.173175936147984 1.40637415801244
-0.161232768137778 1.51343153964938
-0.149289600127573 1.62986592856317
-0.137346432117367 1.7586729581225
-0.125403264107161 1.90312023798198
-0.113460096096955 2.06603815129214
-0.101516928086749 2.24880873327017
-0.0895737600765434 2.45019475200622
-0.0776305920663376 2.66526231257067
-0.0656874240561318 2.88471926320514
-0.053744256045926 3.094979482668
-0.0418010880357202 3.27914927242107
-0.0298579200255145 3.41892816651368
-0.0179147520153087 3.49717000572174
-0.00597158400510289 3.50063395811372
0.00597158400510289 3.4223440206303
0.0179147520153087 3.26301754164879
0.0298579200255145 3.0312185649053
0.0418010880357202 2.74218826589425
0.053744256045926 2.41561480626544
0.065687424056132 2.07283656884217
0.0776305920663378 1.73406246595818
0.0895737600765436 1.41612667771655
0.101516928086749 1.13110875961778
0.113460096096955 0.885911817456364
0.125403264107161 0.682674567053947
0.137346432117367 0.519750907886419
0.149289600127573 0.392943450606355
0.161232768137778 0.296713601114322
0.173175936147984 0.225178369386261
0.18511910415819 0.172806102957705
0.197062272168396 0.1348104131398
0.209005440178601 0.107297919533106
0.220948608188807 0.0872488403325509
0.232891776199013 0.0724069138280706
0.244834944209219 0.0611375058581759
0.256778112219425 0.0522903354427697
0.268721280229631 0.0450833200331298
0.280664448239836 0.0390102593063495
0.292607616250042 0.0337680572056018
0.304550784260248 0.0291975881227258
0.316493952270454 0.0252339522156183
0.32843712028066 0.0218645509047218
0.340380288290865 0.0190955187260524
0.352323456301071 0.0169277545543607
0.364266624311277 0.0153430902253707
0.376209792321483 0.0142995983378739
0.388152960331688 0.01373351261725
0.400096128341894 0.0135644498732564
0.4120392963521 0.0137009353035819
0.423982464362306 0.014044498965212
0.435925632372512 0.0144922763895347
0.447868800382718 0.0149393942925689
0.459811968392923 0.0152828854177094
0.471755136403129 0.015428279723722
0.483698304413335 0.0152986536098111
0.495641472423541 0.0148444123803505
0.507584640433747 0.014051122146729
0.519527808443952 0.0129427469044141
0.531470976454158 0.0115787111966224
0.543414144464364 0.0100449043873936
0.55535731247457 0.00844043781138674
0.567300480484775 0.00686306357347179
0.579243648494981 0.00539633307310044
0.591186816505187 0.00410084946847866
0.603129984515393 0.00301069099969842
0.615073152525599 0.00213472432253667
0.627016320535805 0.0014615020850595
0.63895948854601 0.000965968369621687
0.650902656556216 0.00061627697313421
0.662845824566422 0.000379491738131416
0.674788992576628 0.000225550684677895
0.686732160586834 0.000129431206425859
0.698675328597039 7.18312906285825e-05
0.710618496607245 3.88538211191176e-05
0.722561664617451 2.11875664849826e-05
0.734504832627657 1.31894575406158e-05
0.746448000637862 1.21552720815369e-05
0.758391168648068 1.79644957829666e-05
0.770334336658274 3.32161456893754e-05
0.78227750466848 6.39266150599497e-05
0.794220672678686 0.000120811137667665
0.806163840688891 0.000221082172174071
0.818107008699097 0.000390540189230235
0.830050176709303 0.000665495099348118
0.841993344719509 0.00109376943384736
0.853936512729714 0.00173378078747943
0.86587968073992 0.00265061674813982
0.877822848750126 0.00390825912567902
0.889766016760332 0.00555780460268174
0.901709184770538 0.00762266221433425
0.913652352780744 0.0100830877459451
0.92559552079095 0.0128636402741378
0.937538688801155 0.015827695819562
0.949481856811361 0.0187825681029008
0.961425024821567 0.0214968941209797
0.973368192831773 0.0237290270250816
0.985311360841979 0.0252619931359689
0.997254528852184 0.0259381359950212
1.00919769686239 0.0256858168427732
1.0211408648726 0.0245319159573298
1.0330840328828 0.0225971164293538
1.04502720089301 0.0200751152247774
1.05697036890321 0.0172007165178207
1.06891353691342 0.0142140715877311
1.08085670492362 0.011328539589669
1.09279987293383 0.0087078877755139
1.10474304094404 0.00645557882373132
1.11668620895424 0.00461573664254804
1.12862937696445 0.00318295371329886
1.14057254497465 0.00211691340494307
1.15251571298486 0.00135787355754972
1.16445888099507 0.000840038203924354
1.17640204900527 0.000501212868103744
1.18834521701548 0.000288422313189715
};
\addlegendentry{5\%}
\addplot [semithick, darkorange25512714]
table {%
-1.21789918975618 0.00278153571678537
-1.20565899689431 0.00455104950774472
-1.19341880403244 0.00723828315468062
-1.18117861117057 0.0111908111015044
-1.1689384183087 0.0168188213996232
-1.15669822544683 0.0245722527935862
-1.14445803258495 0.0348994881896121
-1.13221783972308 0.0481869471617241
-1.11997764686121 0.0646830847228156
-1.10773745399934 0.084415409056398
-1.09549726113747 0.107114080340675
-1.0832570682756 0.132158953696933
-1.07101687541373 0.158567063529898
-1.05877668255186 0.185033485364038
-1.04653648968999 0.210030217546298
-1.03429629682811 0.231956438451386
-1.02205610396624 0.249321627963939
-1.00981591110437 0.260933636134231
-0.997575718242501 0.266059637194062
-0.98533552538063 0.264530661466872
-0.973095332518758 0.256769888916875
-0.960855139656887 0.243739065346461
-0.948614946795016 0.226812862774433
-0.936374753933145 0.207603995158571
-0.924134561071274 0.187769433511589
-0.911894368209403 0.16882876802802
-0.899654175347531 0.152020165682806
-0.88741398248566 0.138209557431938
-0.875173789623789 0.127857514872262
-0.862933596761918 0.12103843645343
-0.850693403900047 0.117499960739742
-0.838453211038175 0.11674755221286
-0.826213018176304 0.118139476993215
-0.813972825314433 0.120979776214847
-0.801732632452562 0.124600121138832
-0.789492439590691 0.128424669869783
-0.77725224672882 0.132014795451702
-0.765012053866948 0.135092791092094
-0.752771861005077 0.137545557163813
-0.740531668143206 0.139410986297975
-0.728291475281335 0.140851265665663
-0.716051282419464 0.142118406874657
-0.703811089557593 0.143517716255959
-0.691570896695721 0.14537443451498
-0.67933070383385 0.148007414748806
-0.667090510971979 0.151711730794088
-0.654850318110108 0.156749963760432
-0.642610125248237 0.163350111777829
-0.630369932386366 0.171707018711739
-0.618129739524494 0.181984117730394
-0.605889546662623 0.194313067414575
-0.593649353800752 0.20879022632426
-0.581409160938881 0.225470444059303
-0.56916896807701 0.244359919474096
-0.556928775215139 0.265410576500515
-0.544688582353267 0.288518397097229
-0.532448389491396 0.313527476380817
-0.520208196629525 0.340240422940891
-0.507968003767654 0.368434400885024
-0.495727810905783 0.397880903873312
-0.483487618043911 0.428366533042932
-0.47124742518204 0.459711803193827
-0.459007232320169 0.491785384489988
-0.446767039458298 0.524512112233082
-0.434526846596427 0.55787433233821
-0.422286653734556 0.59190736254243
-0.410046460872684 0.626690696768414
-0.397806268010813 0.66233682514475
-0.385566075148942 0.698979150119289
-0.373325882287071 0.736759649278341
-0.3610856894252 0.775816040126307
-0.348845496563329 0.816267646319555
-0.336605303701457 0.858199225597322
-0.324365110839586 0.901642738658777
-0.312124917977715 0.946558222667944
-0.299884725115844 0.992816261642102
-0.287644532253973 1.04018571824736
-0.275404339392101 1.08833122843121
-0.26316414653023 1.13682538782214
-0.250923953668359 1.18518047689572
-0.238683760806488 1.23290369310783
-0.226443567944617 1.27957764615321
-0.214203375082745 1.32496368777557
-0.201963182220874 1.36911911668688
-0.189722989359003 1.41251079822288
-0.177482796497132 1.45609878066926
-0.165242603635261 1.50135677994452
-0.15300241077339 1.55019538781678
-0.140762217911518 1.6047617611358
-0.128522025049647 1.66710808670157
-0.116281832187776 1.73874925975331
-0.104041639325905 1.82016348773939
-0.0918014464640338 1.91032028290478
-0.0795612536021626 2.00633917691136
-0.0673210607402914 2.10338092058342
-0.0550808678784203 2.19484612471314
-0.0428406750165491 2.27290567641617
-0.0306004821546779 2.32932141921902
-0.0183602892928068 2.3564490140387
-0.00612009643093558 2.34826455553333
0.00612009643093558 2.30123699433264
0.0183602892928068 2.21488726815167
0.0306004821546779 2.09192993137963
0.0428406750165491 1.93797184807837
0.0550808678784203 1.76082644532394
0.0673210607402914 1.56957114523078
0.0795612536021626 1.37351440633266
0.0918014464640338 1.18124043162215
0.104041639325905 0.999867093206021
0.116281832187776 0.834597196386138
0.128522025049647 0.688580251825573
0.140762217911518 0.563046333478136
0.15300241077339 0.457635937936166
0.165242603635261 0.370834339119115
0.177482796497132 0.300423999600517
0.189722989359003 0.243888155874825
0.201963182220875 0.19872501778949
0.214203375082746 0.162658000608953
0.226443567944617 0.133748063420523
0.238683760806488 0.110427237501552
0.250923953668359 0.0914778146577058
0.26316414653023 0.0759810068663278
0.275404339392102 0.063254431068953
0.287644532253973 0.0527917086393777
0.299884725115844 0.0442114996035977
0.312124917977715 0.0372184712523715
0.324365110839586 0.0315754781067169
0.336605303701457 0.0270845936186841
0.348845496563329 0.023574282125447
0.3610856894252 0.0208904931426841
0.373325882287071 0.0188903392320191
0.385566075148942 0.0174378866209828
0.397806268010813 0.0164021659608693
0.410046460872684 0.0156576680347017
0.422286653734556 0.0150873443590755
0.434526846596427 0.0145876268745506
0.446767039458298 0.0140744265808598
0.459007232320169 0.013488684367556
0.47124742518204 0.0127999841418688
0.483487618043911 0.0120070497078842
0.495727810905783 0.0111345684100466
0.507968003767654 0.0102265615558751
0.520208196629525 0.0093372594179208
0.532448389491396 0.00852096089346882
0.544688582353267 0.0078225535831987
0.556928775215139 0.00727021556613606
0.56916896807701 0.00687137431762693
0.581409160938881 0.00661237636956595
0.593649353800752 0.00646166060353982
0.605889546662623 0.00637565553322518
0.618129739524494 0.0063062312688809
0.630369932386366 0.00620838138156191
0.642610125248237 0.00604689445380239
0.654850318110108 0.00580106407572807
0.667090510971979 0.00546691075025317
0.67933070383385 0.00505686077189692
0.691570896695721 0.00459725367198526
0.703811089557593 0.00412435527938377
0.716051282419464 0.0036796944999722
0.728291475281335 0.00330551670415035
0.740531668143206 0.00304099166345459
0.752771861005077 0.00291958927370686
0.765012053866948 0.00296780316275028
0.77725224672882 0.00320520228744521
0.789492439590691 0.00364563411388272
0.801732632452562 0.00429927243859834
0.813972825314433 0.00517507078203577
0.826213018176305 0.00628303618083564
0.838453211038176 0.00763560218228167
0.850693403900047 0.00924731992844926
0.862933596761918 0.0111321930401893
0.875173789623789 0.0132983322968231
0.88741398248566 0.0157402150505108
0.899654175347532 0.0184296185029107
0.911894368209403 0.0213070673535348
0.924134561071274 0.0242761395124091
0.936374753933145 0.0272029628502671
0.948614946795016 0.0299225701449842
0.960855139656887 0.0322524995718507
0.973095332518759 0.0340123832061279
0.98533552538063 0.0350466659028471
0.997575718242501 0.0352464991750724
1.00981591110437 0.034566619228236
1.02205610396624 0.0330337808364127
1.03429629682811 0.030744931649625
1.04653648968999 0.02785537789777
1.05877668255186 0.0245591844172577
1.07101687541373 0.0210654720681672
1.0832570682756 0.0175748045740725
1.09549726113747 0.0142594365953054
1.10773745399934 0.0112500235813324
1.11997764686121 0.00862984169152169
1.13221783972308 0.00643604812509873
1.14445803258495 0.0046663670313553
1.15669822544683 0.00328899653565683
1.1689384183087 0.00225350704011496
1.18117861117057 0.00150091045242055
1.19341880403244 0.000971723946845455
1.20565899689431 0.000611529082714181
1.21789918975618 0.000374087079183632
};
\addlegendentry{10\%}
\addplot [semithick, forestgreen4416044]
table {%
-1.25037753162466 0.0029864161659148
-1.23781092326662 0.00464494722203629
-1.22524431490859 0.00706346555565821
-1.21267770655055 0.01050193867023
-1.20011109819251 0.0152667035350491
-1.18754448983448 0.021699846327183
-1.17497788147644 0.0301590328744787
-1.1624112731184 0.0409868162541459
-1.14984466476037 0.0544699796627206
-1.13727805640233 0.0707915118392073
-1.12471144804429 0.0899800814594169
-1.11214483968626 0.111863934816092
-1.09957823132822 0.136037447549954
-1.08701162297018 0.161848594106995
-1.07444501461215 0.188414002790488
-1.06187840625411 0.214664993687951
-1.04931179789607 0.23942340292944
-1.03674518953804 0.261500822495823
-1.02417858118 0.279810144181002
-1.01161197282196 0.293475055747553
-0.999045364463925 0.301922248158646
-0.986478756105889 0.304942948966952
-0.973912147747852 0.302714793164179
-0.961345539389815 0.295781184507771
-0.948778931031778 0.284991973721832
-0.936212322673742 0.271415140230388
-0.923645714315705 0.256233062612431
-0.911079105957668 0.240638220516398
-0.898512497599631 0.225741720628997
-0.885945889241594 0.212504405376012
-0.873379280883558 0.201695428403295
-0.860812672525521 0.19387816534715
-0.848246064167484 0.189419155417168
-0.835679455809447 0.188513097287615
-0.82311284745141 0.191215987984734
-0.810546239093374 0.19747913303758
-0.797979630735337 0.207178524153888
-0.7854130223773 0.220136407794922
-0.772846414019263 0.236134213404146
-0.760279805661226 0.254917966612513
-0.74771319730319 0.276198650306026
-0.735146588945153 0.299650629373892
-0.722579980587116 0.324911276183339
-0.710013372229079 0.351584440894055
-0.697446763871042 0.379249543852477
-0.684880155513006 0.407476967920163
-0.672313547154969 0.435849235343575
-0.659746938796932 0.463986306036881
-0.647180330438895 0.491572372225437
-0.634613722080858 0.518380880246527
-0.622047113722822 0.544294289446698
-0.609480505364785 0.569315338902364
-0.596913897006748 0.59356733100368
-0.584347288648711 0.617282086142771
-0.571780680290674 0.640775645993509
-0.559214071932638 0.664413334405298
-0.546647463574601 0.688567234934474
-0.534080855216564 0.713570322769161
-0.521514246858527 0.73967222503902
-0.50894763850049 0.767001743060047
-0.496381030142454 0.795540776653701
-0.483814421784417 0.825113148286753
-0.47124781342638 0.855390135973602
-0.458681205068343 0.885912494482706
-0.446114596710306 0.916126665022096
-0.43354798835227 0.945431074245917
-0.420981379994233 0.973227206939863
-0.408414771636196 0.998969707257896
-0.395848163278159 1.02221017223604
-0.383281554920122 1.042630430554
-0.370714946562086 1.06006269292343
-0.358148338204049 1.07449569571205
-0.345581729846012 1.08606753924775
-0.333015121487975 1.09504715374938
-0.320448513129938 1.10180715849715
-0.307881904771902 1.10679138741035
-0.295315296413865 1.11048066560008
-0.282748688055828 1.11336062725696
-0.270182079697791 1.11589543794355
-0.257615471339754 1.11851105047169
-0.245048862981718 1.12159080901282
-0.232482254623681 1.12548456152647
-0.219915646265644 1.13052985421264
-0.207349037907607 1.13708047748908
-0.19478242954957 1.14553419603621
-0.182215821191534 1.15634882972115
-0.169649212833497 1.1700349970907
-0.15708260447546 1.18711567579051
-0.144515996117423 1.20804768563937
-0.131949387759386 1.2331079290109
-0.11938277940135 1.26225654661201
-0.106816171043313 1.29499814652593
-0.094249562685276 1.33026864617594
-0.0816829543272393 1.36637690538163
-0.0691163459692024 1.4010258744105
-0.0565497376111657 1.43142736060263
-0.043983129253129 1.45450916243734
-0.0314165208950921 1.46719601446984
-0.0188499125370554 1.46673012158777
-0.00628330417901846 1.4509866078692
0.00628330417901823 1.41873661280622
0.0188499125370551 1.36981703589018
0.0314165208950918 1.30518008019696
0.0439831292531287 1.22681499679989
0.0565497376111654 1.1375548590432
0.0691163459692024 1.04079868210614
0.081682954327239 0.940190428210363
0.094249562685276 0.839299624467729
0.106816171043313 0.741343549713195
0.11938277940135 0.64897997951412
0.131949387759386 0.5641851735884
0.144515996117423 0.48821731386881
0.15708260447546 0.421653666016809
0.169649212833497 0.364482060010464
0.182215821191533 0.316224383012112
0.19478242954957 0.276071087066571
0.207349037907607 0.24300995027743
0.219915646265644 0.215937938984127
0.232482254623681 0.193750588603026
0.245048862981718 0.175407857740796
0.257615471339754 0.159978438854125
0.270182079697791 0.146666023345067
0.282748688055828 0.134821335926446
0.295315296413865 0.123943332123457
0.307881904771901 0.113672241090552
0.320448513129938 0.103776466669772
0.333015121487975 0.0941349121768753
0.345581729846012 0.0847161053713357
0.358148338204049 0.0755555049090953
0.370714946562085 0.0667324544046074
0.383281554920122 0.0583482982510846
0.395848163278159 0.050507097538133
0.408414771636196 0.0433001441302153
0.420981379994233 0.0367950743849081
0.433547988352269 0.0310298798783232
0.446114596710306 0.026011575556424
0.458681205068343 0.0217187989624674
0.47124781342638 0.0181072517120657
0.483814421784417 0.0151167072656168
0.496381030142453 0.0126783164507316
0.50894763850049 0.0107211291114537
0.521514246858527 0.00917707196530156
0.534080855216564 0.0079840146895805
0.546647463574601 0.00708694721341643
0.559214071932638 0.00643761669547694
0.571780680290674 0.00599318601207667
0.584347288648711 0.00571455443927003
0.596913897006748 0.00556492908373718
0.609480505364785 0.00550907846659141
0.622047113722821 0.00551347901652457
0.634613722080858 0.00554732939340032
0.647180330438895 0.00558420269080314
0.659746938796932 0.00560396895167583
0.672313547154969 0.00559457107466816
0.684880155513006 0.00555327918362027
0.697446763871042 0.00548716745897423
0.710013372229079 0.00541272455003908
0.722579980587116 0.00535468640094778
0.735146588945153 0.00534432906527034
0.747713197303189 0.00541754466485422
0.760279805661226 0.00561302399447549
0.772846414019263 0.0059707795809474
0.7854130223773 0.00653107842574725
0.797979630735337 0.00733364878473419
0.810546239093374 0.00841682991724106
0.82311284745141 0.00981620444258332
0.835679455809447 0.0115622416500785
0.848246064167484 0.0136766201972247
0.860812672525521 0.0161671920299588
0.873379280883557 0.019021958061547
0.885945889241594 0.0222028720496007
0.898512497599631 0.0256406648772809
0.911079105957668 0.029232072433236
0.923645714315704 0.0328407639742101
0.936212322673741 0.0363028637046049
0.948778931031778 0.0394372690263429
0.961345539389815 0.0420601048386257
0.973912147747852 0.0440017854014653
0.986478756105889 0.0451244786268173
0.999045364463925 0.0453374531661603
1.01161197282196 0.0446079367033352
1.02417858118 0.0429657232551229
1.03674518953804 0.0405007300039547
1.04931179789607 0.0373538286783198
1.06187840625411 0.0337023337529484
1.07444501461215 0.029742309298153
1.08701162297018 0.0256702162579496
1.09957823132822 0.0216663192381036
1.11214483968626 0.0178817656252577
1.12471144804429 0.0144304785015166
1.13727805640233 0.0113861459296397
1.14984466476037 0.00878381428487581
1.1624112731184 0.00662503040899731
1.17497788147644 0.00488518980419199
1.18754448983448 0.0035217324859866
1.20011109819251 0.00248202842777935
1.21267770655055 0.00171012544854681
1.22524431490859 0.00115190422218262
1.23781092326662 0.000758523570988903
1.25037753162466 0.000488296760272226
};
\addlegendentry{20\%}
\addplot [semithick, crimson2143940]
table {%
-1.25838020593482 0.0032105877247178
-1.24573316868924 0.00495016247179147
-1.23308613144367 0.00747194308964997
-1.22043909419809 0.0110420038399659
-1.20779205695251 0.0159767068950163
-1.19514501970694 0.0226350273239032
-1.18249798246136 0.0314024867413764
-1.16985094521579 0.042665611443074
-1.15720390797021 0.056777046424695
-1.14455687072463 0.0740131273786564
-1.13190983347906 0.0945276427940255
-1.11926279623348 0.11830737681389
-1.10661575898791 0.145136384663726
-1.09396872174233 0.174576376020346
-1.08132168449675 0.205969722626793
-1.06867464725118 0.238469328762737
-1.0560276100056 0.271096064883349
-1.04338057276003 0.302820142189617
-1.03073353551445 0.332658437284916
-1.01808649826887 0.359776225511531
-1.0054394610233 0.383579849644048
-0.992792423777722 0.40378707759883
-0.980145386532146 0.420464418214214
-0.96749834928657 0.434025127993334
-0.954851312040993 0.445187299000061
-0.942204274795417 0.454897256367651
-0.929557237549841 0.464228463792104
-0.916910200304265 0.474269377016717
-0.904263163058689 0.486014702260895
-0.891616125813113 0.500273257257166
-0.878969088567537 0.517602466275994
-0.866322051321961 0.538275124230269
-0.853675014076385 0.562279258972109
-0.841027976830809 0.589347498108238
-0.828380939585233 0.619008934278279
-0.815733902339657 0.650654459819309
-0.803086865094081 0.683606016742054
-0.790439827848505 0.717181047973518
-0.777792790602928 0.750745327893096
-0.765145753357352 0.783749877635776
-0.752498716111776 0.815750389435852
-0.7398516788662 0.846410089855814
-0.727204641620624 0.875488950524364
-0.714557604375048 0.902823417232428
-0.701910567129472 0.928301319475487
-0.689263529883896 0.951836413135577
-0.67661649263832 0.973346263620591
-0.663969455392744 0.992736110124047
-0.651322418147168 1.00989018020572
-0.638675380901592 1.02467082473603
-0.626028343656016 1.03692492949142
-0.61338130641044 1.0464963766022
-0.600734269164863 1.05324286901524
-0.588087231919287 1.05705515953864
-0.575440194673711 1.05787660756705
-0.562793157428135 1.05572100226703
-0.550146120182559 1.05068674089598
-0.537499082936983 1.04296574347526
-0.524852045691407 1.03284591939122
-0.512205008445831 1.02070655198021
-0.499557971200255 1.00700657793958
-0.486910933954679 0.992266332236229
-0.474263896709103 0.977043827946937
-0.461616859463527 0.961906990107394
-0.448969822217951 0.947403452237752
-0.436322784972375 0.93402959029623
-0.423675747726798 0.922200481246784
-0.411028710481222 0.912222505943247
-0.398381673235646 0.904270411328396
-0.38573463599007 0.898370789484557
-0.373087598744494 0.894394039380235
-0.360440561498918 0.892056821363163
-0.347793524253342 0.890936655948078
-0.335146487007766 0.890499560108705
-0.32249944976219 0.890140444718279
-0.309852412516614 0.889234512902276
-0.297205375271038 0.887196297626807
-0.284558338025462 0.883541515749928
-0.271911300779886 0.877945856098208
-0.259264263534309 0.870294365554974
-0.246617226288733 0.860715354841793
-0.233970189043157 0.849593706258273
-0.221323151797581 0.837560024189637
-0.208676114552005 0.825454063454356
-0.196029077306429 0.814263125721589
-0.183382040060853 0.805038474839249
-0.170735002815277 0.798795158387559
-0.158087965569701 0.796402810583295
-0.145440928324125 0.798476894199608
-0.132793891078549 0.805281194716859
-0.120146853832973 0.816652915707484
-0.107499816587397 0.831961114159728
-0.0948527793418206 0.85010718095145
-0.0822057420962445 0.869572498489558
-0.0695587048506683 0.888513446032856
-0.0569116676050923 0.904898055811571
-0.0442646303595162 0.91667264742085
-0.0316175931139402 0.921941686827988
-0.0189705558683642 0.919140944869102
-0.0063235186227879 0.907183561700778
0.00632351862278813 0.885561214407913
0.0189705558683642 0.854388020089461
0.0316175931139402 0.814382295143583
0.0442646303595162 0.766789601245416
0.0569116676050925 0.713258216666375
0.0695587048506685 0.655683967220886
0.0822057420962445 0.596044314530226
0.0948527793418206 0.536241385602923
0.107499816587397 0.477970511016606
0.120146853832973 0.422625612086188
0.132793891078549 0.371246541984853
0.145440928324125 0.324507393797529
0.158087965569701 0.282739803257576
0.170735002815277 0.245982006182456
0.183382040060853 0.214043057896853
0.196029077306429 0.186572008501502
0.208676114552005 0.16312351719457
0.221323151797581 0.143213826998478
0.233970189043157 0.126363670197479
0.246617226288733 0.112127107104689
0.259264263534309 0.100107243409407
0.271911300779886 0.089961103443452
0.284558338025462 0.0813966568301787
0.297205375271038 0.0741651780307307
0.309852412516614 0.0680518719213875
0.32249944976219 0.062867142982498
0.335146487007766 0.0584401380643143
0.347793524253342 0.0546153657124371
0.360440561498918 0.0512523970317002
0.373087598744494 0.0482279845661422
0.38573463599007 0.0454394801150768
0.398381673235646 0.0428082430255918
0.411028710481222 0.0402818182241907
0.423675747726799 0.0378339911495046
0.436322784972375 0.0354623148558454
0.448969822217951 0.0331832453853506
0.461616859463527 0.0310255029408444
0.474263896709103 0.0290226059857892
0.486910933954679 0.027205649474321
0.499557971200255 0.0255973109619827
0.512205008445831 0.0242078078494897
0.524852045691407 0.0230331649329251
0.537499082936983 0.0220557649150503
0.550146120182559 0.0212468190163408
0.562793157428135 0.0205701616470564
0.575440194673711 0.0199866649460814
0.588087231919287 0.0194585813809097
0.600734269164864 0.0189532312806937
0.61338130641044 0.0184456229925089
0.626028343656016 0.0179197913473861
0.638675380901592 0.0173688356117655
0.651322418147168 0.0167938091773839
0.663969455392744 0.0162017454809894
0.67661649263832 0.0156031894820861
0.689263529883896 0.0150096372783706
0.701910567129472 0.014431267738134
0.714557604375048 0.0138752831682984
0.727204641620624 0.0133450695820982
0.7398516788662 0.0128402545185887
0.752498716111776 0.0123575989478969
0.765145753357352 0.0118925287888957
0.777792790602929 0.0114410093433922
0.790439827848505 0.0110014073233961
0.803086865094081 0.0105759792243942
0.815733902339657 0.0101716737484127
0.828380939585233 0.00980003472327237
0.841027976830809 0.00947612760809282
0.853675014076385 0.00921656956624587
0.866322051321961 0.00903689842231841
0.878969088567537 0.00894864592040467
0.891616125813113 0.00895656303346135
0.904263163058689 0.00905646189670984
0.916910200304265 0.00923408110832603
0.929557237549842 0.00946525124314372
0.942204274795417 0.00971745122757683
0.954851312040994 0.00995263158993765
0.967498349286569 0.0101309735073545
0.980145386532146 0.0102150906480605
0.992792423777722 0.0101740956186854
1.0054394610233 0.00998696268931171
1.01808649826887 0.00964472372663308
1.03073353551445 0.00915121642914058
1.04338057276003 0.00852232921724493
1.0560276100056 0.00778391353027886
1.06867464725118 0.00696872083658101
1.08132168449675 0.00611283711415879
1.09396872174233 0.00525211657733349
1.10661575898791 0.00441906179597936
1.11926279623348 0.00364047784404667
1.13190983347906 0.00293607286219776
1.14455687072463 0.00231801857767436
1.15720390797021 0.00179134994391133
1.16985094521579 0.00135499192386326
1.18249798246136 0.00100316078963992
1.19514501970694 0.000726893552032952
1.20779205695252 0.000515500778512109
1.22043909419809 0.000357799773835655
1.23308613144367 0.000243051666859485
1.24573316868924 0.000161585244652565
1.25838020593482 0.000105134912625627
};
\addlegendentry{40\%}
\end{axis}

\end{tikzpicture}

%% file: figures/attackM_rpd_eps.tex
\begin{tikzpicture}

\definecolor{crimson2143940}{RGB}{214,39,40}
\definecolor{darkgray176}{RGB}{176,176,176}
\definecolor{darkorange25512714}{RGB}{255,127,14}
\definecolor{forestgreen4416044}{RGB}{44,160,44}
\definecolor{lightgray204}{RGB}{204,204,204}
\definecolor{steelblue31119180}{RGB}{31,119,180}

\begin{axis}[
legend cell align={left},
legend style={fill opacity=0.8, draw opacity=1, text opacity=1, draw=lightgray204},
tick align=outside,
tick pos=left,
x grid style={darkgray176},
xlabel={RPD of EPS},
xmin=-3.00844414156937, xmax=3.00844215868248,
xtick style={color=black},
y grid style={darkgray176},
ylabel={Density},
ymin=0, ymax=1.10939055650551,
ytick style={color=black},
width=\columnwidth,
height=0.5\columnwidth
]
\addplot [semithick, steelblue31119180]
table {%
-2.69952314173629 0.00439009076846125
-2.67239226203336 0.00618672348778598
-2.64526138233042 0.00860180919433492
-2.61813050262748 0.0117994810637801
-2.59099962292454 0.0159691721560462
-2.5638687432216 0.0213231322162017
-2.53673786351867 0.0280912963090442
-2.50960698381573 0.0365130626129346
-2.48247610411279 0.0468257202310052
-2.45534522440985 0.0592495435333846
-2.42821434470691 0.0739699289443185
-2.40108346500398 0.0911173695422441
-2.37395258530104 0.110746502655342
-2.3468217055981 0.132815872103249
-2.31969082589516 0.157170357184902
-2.29255994619222 0.18352837125214
-2.26542906648928 0.211475869084833
-2.23829818678635 0.240468889420342
-2.21116730708341 0.269845791341628
-2.18403642738047 0.298849550900645
-2.15690554767753 0.326659534428179
-2.12977466797459 0.352431156735048
-2.10264378827166 0.37534088637298
-2.07551290856872 0.39463330211393
-2.04838202886578 0.409666447213783
-2.02125114916284 0.41995165142561
-1.9941202694599 0.425184329148414
-1.96698938975697 0.425262994867722
-1.93985851005403 0.420294789088458
-1.91272763035109 0.410587058939401
-1.88559675064815 0.396625838661244
-1.85846587094521 0.379043270436437
-1.83133499124228 0.358576955443832
-1.80420411153934 0.336024824430927
-1.7770732318364 0.312199312062706
-1.74994235213346 0.287884410789856
-1.72281147243052 0.263798621045646
-1.69568059272759 0.240565999244029
-1.66854971302465 0.218696551632182
-1.64141883332171 0.198576254255821
-1.61428795361877 0.180466108356181
-1.58715707391583 0.16450895122364
-1.56002619421289 0.150742285004051
-1.53289531450996 0.139115173764782
-1.50576443480702 0.129507273206072
-1.47863355510408 0.121748254196991
-1.45150267540114 0.115636203205881
-1.4243717956982 0.110953968584743
-1.39724091599527 0.107482815988354
-1.37011003629233 0.105013115218265
-1.34297915658939 0.103352075683487
-1.31584827688645 0.10232876477235
-1.28871739718351 0.101796782006684
-1.26158651748058 0.101635030622138
-1.23445563777764 0.101747041410102
-1.2073247580747 0.10205927738957
-1.18019387837176 0.102518797428722
-1.15306299866882 0.103090595057187
-1.12593211896589 0.10375486495469
-1.09880123926295 0.10450439042189
-1.07167035956001 0.10534219442634
-1.04453947985707 0.106279556570686
-1.01740860015413 0.107334469218152
-0.990277720451195 0.108530587681642
-0.963146840748257 0.109896720532151
-0.936015961045319 0.111466904333257
-0.908885081342381 0.113281108853615
-0.881754201639443 0.115386619149104
-0.854623321936505 0.117840133813847
-0.827492442233567 0.120710597487089
-0.800361562530629 0.12408274381136
-0.77323068282769 0.128061257111426
-0.746099803124753 0.132775364227849
-0.718968923421814 0.138383543142632
-0.691838043718876 0.145077888292758
-0.664707164015938 0.153087515795189
-0.637576284313 0.162680243598262
-0.610445404610062 0.17416166617197
-0.583314524907124 0.187870689547062
-0.556183645204186 0.204170630812647
-0.529052765501248 0.223435144756408
-0.50192188579831 0.246028539880024
-0.474791006095372 0.272280493996991
-0.447660126392434 0.302455764641148
-0.420529246689496 0.336720177066844
-0.393398366986558 0.375104903289919
-0.366267487283619 0.417471737025666
-0.339136607580681 0.463482622640911
-0.312005727877743 0.512577007047433
-0.284874848174805 0.563960557031832
-0.257743968471867 0.616608352561273
-0.230613088768929 0.669284803734536
-0.203482209065991 0.720581275371763
-0.176351329363053 0.768970830069126
-0.149220449660115 0.812877766867352
-0.122089569957177 0.850757931366803
-0.0949586902542388 0.881184317811125
-0.0678278105513006 0.902931479128919
-0.0406969308483625 0.915051871387503
-0.0135660511454243 0.91693757991529
0.0135648285575138 0.908361912314682
0.0406957082604515 0.889497010221218
0.0678265879633897 0.860905751228011
0.0949574676663278 0.823508544237433
0.122088347369266 0.778527893354045
0.149219227072204 0.72741555223384
0.176350106775142 0.67176849253915
0.20348098647808 0.61324062230156
0.230611866181018 0.553457161654538
0.257742745883956 0.493937860344702
0.284873625586894 0.436033954337643
0.312004505289833 0.380882099340026
0.339135384992771 0.329376708425736
0.366266264695708 0.282160377005052
0.393397144398647 0.239630586191883
0.420528024101585 0.201959766314978
0.447658903804523 0.169125142318928
0.474789783507461 0.140944574057449
0.501920663210399 0.117114794000306
0.529051542913337 0.0972489407328054
0.556182422616275 0.0809109769244053
0.583313302319213 0.0676453517606253
0.610444182022151 0.0570010204286056
0.637575061725089 0.0485495912916439
0.664705941428028 0.0418978876642424
0.691836821130965 0.0366955666181438
0.718967700833903 0.0326386369609751
0.746098580536842 0.0294697846655056
0.77322946023978 0.0269763786559489
0.800360339942718 0.0249869279535431
0.827491219645656 0.0233666251439583
0.854622099348594 0.0220124671766604
0.881752979051532 0.0208483110201659
0.90888385875447 0.0198201091900158
0.936014738457408 0.0188914824439282
0.963145618160346 0.0180397229141141
0.990276497863285 0.0172522766272452
1.01740737756622 0.0165237246328822
1.04453825726916 0.0158532619308101
1.0716691369721 0.0152426592070994
1.09880001667504 0.0146946815581858
1.12593089637797 0.0142119296595088
1.15306177608091 0.0137960618886262
1.18019265578385 0.0134473508749474
1.20732353548679 0.013164525002316
1.23445441518973 0.0129448444807283
1.26158529489267 0.0127843623149571
1.2887161745956 0.0126783221070532
1.31584705429854 0.0126216462842906
1.34297793400148 0.0126094693134112
1.37010881370442 0.0126376703852549
1.39723969340736 0.0127033591200945
1.42437057311029 0.0128052668529902
1.45150145281323 0.0129439963368596
1.47863233251617 0.0131220858766215
1.50576321221911 0.0133438515741783
1.53289409192205 0.0136149846898692
1.56002497162498 0.0139419004817851
1.58715585132792 0.014330859561334
1.61428673103086 0.0147869108932116
1.6414176107338 0.0153127340483112
1.66854849043674 0.0159074833608222
1.69567937013967 0.01656575411714
1.72281024984261 0.0172767970263313
1.74994112954555 0.0180240992242293
1.77707200924849 0.0187854268091847
1.80420288895143 0.019533386304938
1.83133376865436 0.0202365135754028
1.8584646483573 0.0208608436150515
1.88559552806024 0.0213718597438577
1.91272640776318 0.0217366730433672
1.93985728746612 0.0219262489331433
1.96698816716906 0.0219174827062382
1.99411904687199 0.0216949324278033
2.02124992657493 0.0212520458464044
2.04838080627787 0.0205917649511877
2.07551168598081 0.0197264520464867
2.10264256568375 0.018677147426074
2.12977344538668 0.0174722328775387
2.15690432508962 0.0161456297065906
2.18403520479256 0.0147346986008384
2.2111660844955 0.0132780276179656
2.23829696419844 0.011813292790085
2.26542784390137 0.010375354926257
2.29255872360431 0.00899472012782685
2.31968960330725 0.00769644584572657
2.34682048301019 0.00649952521424541
2.37395136271313 0.00541673584878906
2.40108224241606 0.00445490021623606
2.428213122119 0.00361547642263983
2.45534400182194 0.00289538235217873
2.48247488152488 0.00228795232443523
2.50960576122782 0.00178393216438803
2.53673664093076 0.00137243317385875
2.56386752063369 0.00104178487299887
2.59099840033663 0.000780247500055917
2.61812928003957 0.000576565496322759
2.64526015974251 0.000420360613556472
2.67239103944545 0.000302376689778066
2.69952191914838 0.000214597116308101
};
\addlegendentry{5\%}
\addplot [semithick, darkorange25512714]
table {%
-2.73494930973974 0.00634202448118478
-2.70746239196792 0.00883332643012609
-2.67997547419609 0.0121501127549157
-2.65248855642427 0.0165043383652394
-2.62500163865245 0.0221401238193131
-2.59751472088063 0.0293311212255356
-2.57002780310881 0.0383748622736897
-2.54254088533698 0.0495835602150412
-2.51505396756516 0.0632710179800171
-2.48756704979334 0.0797355718546028
-2.46008013202152 0.0992393690998929
-2.4325932142497 0.121984719050525
-2.40510629647787 0.148088735389058
-2.37761937870605 0.177557952681885
-2.35013246093423 0.210264992213942
-2.32264554316241 0.245929605582968
-2.29515862539059 0.28410647860199
-2.26767170761876 0.324181986238753
-2.24018478984694 0.365381628976373
-2.21269787207512 0.40678916120168
-2.1852109543033 0.447377487901069
-2.15772403653148 0.486050336840876
-2.13023711875965 0.521692617673284
-2.10275020098783 0.553226381735444
-2.07526328321601 0.57966852214064
-2.04777636544419 0.600185911680291
-2.02028944767236 0.614143641129241
-1.99280252990054 0.621142421429232
-1.96531561212872 0.62104202542443
-1.9378286943569 0.613968791573627
-1.91034177658508 0.600306573024417
-1.88285485881325 0.580671942063975
-1.85536794104143 0.555875795261044
-1.82788102326961 0.526874604144612
-1.80039410549779 0.494715307118306
-1.77290718772597 0.460478173111313
-1.74542026995414 0.425221871382859
-1.71793335218232 0.389934491846014
-1.6904464344105 0.355493456891449
-1.66295951663868 0.322636259713618
-1.63547259886686 0.291942879729123
-1.60798568109503 0.263829683400328
-1.58049876332321 0.238553720858452
-1.55301184555139 0.216225648456793
-1.52552492777957 0.196829083798353
-1.49803801000775 0.180244037593361
-1.47055109223592 0.166272140876509
-1.4430641744641 0.15466164963662
-1.41557725669228 0.145130602164005
-1.38809033892046 0.137386964558439
-1.36060342114864 0.131145068861093
-1.33311650337681 0.126138078902478
-1.30562958560499 0.122126577977696
-1.27814266783317 0.118903641556806
-1.25065575006135 0.116296932795971
-1.22316883228953 0.114168445077728
-1.1956819145177 0.112412528298334
-1.16819499674588 0.110952792440301
-1.14070807897406 0.109738402478146
-1.11322116120224 0.108740180833293
-1.08573424343041 0.107946832360021
-1.05824732565859 0.107361513271673
-1.03076040788677 0.106998886504884
-1.00327349011495 0.106882745000706
-0.975786572343127 0.107044241301633
-0.948299654571305 0.107520734295212
-0.920812736799483 0.108355247696483
-0.893325819027661 0.109596524688534
-0.865838901255839 0.111299653298534
-0.838351983484017 0.113527221952065
-0.810865065712194 0.116350939300424
-0.783378147940372 0.119853613266596
-0.75589123016855 0.124131329731347
-0.728404312396728 0.129295602498595
-0.700917394624906 0.135475187540823
-0.673430476853084 0.142817174084195
-0.645943559081262 0.151486894532584
-0.61845664130944 0.161666149290197
-0.590969723537618 0.173549237772708
-0.563482805765796 0.18733633968974
-0.535995887993974 0.203223914672407
-0.508508970222152 0.221391991379403
-0.48102205245033 0.241988498423822
-0.453535134678507 0.265111136585416
-0.426048216906686 0.290787679826076
-0.398561299134863 0.318955984106266
-0.371074381363042 0.349445330192465
-0.343587463591219 0.381960976169795
-0.316100545819397 0.416073894739548
-0.288613628047576 0.451217575904127
-0.261126710275753 0.486693460669161
-0.233639792503931 0.521686033487467
-0.206152874732109 0.555287866839218
-0.178665956960287 0.586534037046626
-0.151179039188465 0.614444399081279
-0.123692121416643 0.638071321012229
-0.0962052036448209 0.656549743479642
-0.0687182858729991 0.669145945715119
-0.0412313681011769 0.675301243951583
-0.0137444503293547 0.674667062007233
0.0137424674424671 0.667128395209578
0.0412293852142893 0.652813589602414
0.0687163029861115 0.632089489092557
0.0962032207579333 0.605542243733852
0.123690138529756 0.573945287504146
0.151177056301577 0.538217051373308
0.1786639740734 0.49937176530685
0.206150891845222 0.458467144100585
0.233637809617044 0.416552813558727
0.261124727388866 0.374623028923002
0.288611645160688 0.333576622281593
0.31609856293251 0.294186277906179
0.343585480704332 0.257078280646823
0.371072398476154 0.222722922632239
0.398559316247976 0.19143488673546
0.426046234019798 0.163382228676814
0.45353315179162 0.138602101375237
0.481020069563442 0.117021122269747
0.508506987335264 0.0984782651447879
0.535993905107086 0.0827483279777784
0.563480822878908 0.0695643380557012
0.590967740650731 0.058637649312079
0.618454658422552 0.0496749102498055
0.645941576194375 0.0423914873100915
0.673428493966196 0.0365212830382124
0.700915411738019 0.0318231689328022
0.728402329509841 0.0280844501945246
0.755889247281663 0.0251218952165516
0.783376165053485 0.0227809060306465
0.810863082825307 0.0209333913454568
0.838350000597129 0.0194748475913786
0.865836918368951 0.0183210715044321
0.893323836140773 0.017404834294103
0.920810753912595 0.016672753564906
0.948297671684417 0.0160825129825888
0.975784589456239 0.0156005062951482
1.00327150722806 0.0151999242830375
1.03075842499988 0.0148592610635214
1.05824534277171 0.0145611890026912
1.08573226054353 0.0142917374275726
1.11321917831535 0.0140397069487251
1.14070609608717 0.0137962558067862
1.16819301385899 0.013554604513398
1.19567993163082 0.0133098175744266
1.22316684940264 0.0130586339509181
1.25065376717446 0.0127993292315867
1.27814068494628 0.012531600882672
1.3056276027181 0.0122564726366942
1.33311452048993 0.0119762149610213
1.36060143826175 0.0116942761258557
1.38808835603357 0.0114152137365971
1.41557527380539 0.0111446111719274
1.44306219157721 0.0108889588218046
1.47054910934904 0.0106554779207758
1.49803602712086 0.0104518663591207
1.52552294489268 0.0102859518013975
1.5530098626645 0.0101652477001809
1.58049678043632 0.0100964215313476
1.60798369820815 0.0100847002424671
1.63547061597997 0.0101332533990995
1.66295753375179 0.010242607470438
1.69044445152361 0.0104101528291695
1.71793136929543 0.0106298065027455
1.74541828706726 0.0108918874174688
1.77290520483908 0.011183246727784
1.8003921226109 0.0114876748153139
1.82787904038272 0.0117865806998363
1.85536595815454 0.0120599117644769
1.88285287592637 0.0122872551588549
1.91033979369819 0.012449040330311
1.93782671147001 0.0125277477319272
1.96531362924183 0.0125090239022082
1.99280054701365 0.012382608683892
2.02028746478548 0.0121429959257357
2.0477743825573 0.0117897729075599
2.07526130032912 0.0113276132220149
2.10274821810094 0.0107659295433873
2.13023513587277 0.0101182230126101
2.15772205364459 0.00940119153971722
2.18520897141641 0.00863367754306411
2.21269588918823 0.0078355448959448
2.24018280696005 0.00702657469244546
2.26766972473188 0.00622546058383783
2.2951566425037 0.0054489685573765
2.32264356027552 0.00471130548983981
2.35013047804734 0.00402371827784009
2.37761739581916 0.00339432343357
2.40510431359099 0.00282814796147135
2.43259123136281 0.00232734770508197
2.46007814913463 0.00189156004622506
2.48756506690645 0.00151834398612467
2.51505198467827 0.00120366173581891
2.5425389024501 0.000942361015308199
2.57002582022192 0.000728625061316462
2.59751273799374 0.000556366558449267
2.62499965576556 0.000419551128685115
2.65248657353738 0.000312444659713545
2.67997349130921 0.000229785942027071
2.70746040908103 0.000166891449572544
2.73494732685285 0.000119702555320999
};
\addlegendentry{10\%}
\addplot [semithick, forestgreen4416044]
table {%
-2.72872681098628 0.0079675351802989
-2.70130242590825 0.011121734129952
-2.67387804083022 0.0153288878297433
-2.64645365575218 0.020861253668447
-2.61902927067415 0.0280326409307484
-2.59160488559611 0.0371949493703045
-2.56418050051808 0.0487307476401761
-2.53675611544005 0.063041199404109
-2.50933173036201 0.0805288871094425
-2.48190734528398 0.101575454943411
-2.45448296020594 0.126514485968549
-2.42705857512791 0.155600616766029
-2.39963419004988 0.188976527834399
-2.37220980497184 0.226640061578216
-2.34478541989381 0.268414229693669
-2.31736103481578 0.313923190491007
-2.28993664973774 0.362577323286876
-2.26251226465971 0.413570240941676
-2.23508787958167 0.465889935957669
-2.20766349450364 0.518345267164588
-2.18023910942561 0.569607728604709
-2.15281472434757 0.618267012307219
-2.12539033926954 0.66289743175552
-2.0979659541915 0.702130982583424
-2.07054156911347 0.734731848811145
-2.04311718403544 0.759666657787589
-2.0156927989574 0.776164837218367
-1.98826841387937 0.783764060139714
-1.96084402880133 0.782336932342707
-1.9334196437233 0.772096665544572
-1.90599525864527 0.75358131646094
-1.87857087356723 0.727618051370189
-1.8511464884892 0.695270606874618
-1.82372210341116 0.657774472844114
-1.79629771833313 0.616465183870512
-1.7688733332551 0.572705397543292
-1.74144894817706 0.527816160782133
-1.71402456309903 0.483016987475445
-1.68660017802099 0.43937821516767
-1.65917579294296 0.397787733588217
-1.63175140786493 0.35893275239052
-1.60432702278689 0.323295956358122
-1.57690263770886 0.291164309247397
-1.54947825263083 0.262647994702669
-1.52205386755279 0.237706558579753
-1.49462948247476 0.216179230435021
-1.46720509739672 0.197816604890508
-1.43978071231869 0.182311283413923
-1.41235632724066 0.16932562990974
-1.38493194216262 0.158515396915609
-1.35750755708459 0.149548561962003
-1.33008317200655 0.142119222343146
-1.30265878692852 0.135956798149374
-1.27523440185049 0.130831074704465
-1.24781001677245 0.126553779991087
-1.22038563169442 0.122977455780235
-1.19296124661638 0.119992365595524
-1.16553686153835 0.117522113285958
-1.13811247646032 0.115518546393069
-1.11068809138228 0.1139564082703
-1.08326370630425 0.112828096375091
-1.05583932122621 0.112138790155435
-1.02841493614818 0.111902134443225
-1.00099055107015 0.112136603379497
-0.973566165992113 0.11286262322018
-0.946141780914079 0.114100496169678
-0.918717395836045 0.115869137605501
-0.891293010758011 0.118185612006186
-0.863868625679977 0.121065425652901
-0.836444240601943 0.12452350471033
-0.809019855523909 0.128575754437719
-0.781595470445875 0.133241058661857
-0.754171085367841 0.138543538659795
-0.726746700289807 0.144514848521803
-0.699322315211774 0.151196242261583
-0.671897930133739 0.158640110164037
-0.644473545055706 0.166910653490303
-0.617049159977672 0.176083354732803
-0.589624774899637 0.186242913482938
-0.562200389821604 0.197479364468469
-0.53477600474357 0.209882182353664
-0.507351619665536 0.223532312636028
-0.479927234587502 0.238492249751641
-0.452502849509468 0.254794505871476
-0.425078464431434 0.272429062247846
-0.3976540793534 0.291330646317733
-0.370229694275366 0.311366901733938
-0.342805309197332 0.332328680165466
-0.315380924119299 0.353923747926372
-0.287956539041264 0.375775137276991
-0.260532153963231 0.397425162601782
-0.233107768885197 0.418345762814221
-0.205683383807163 0.437955340247732
-0.178258998729129 0.455641680677342
-0.150834613651095 0.470789915198182
-0.123410228573061 0.482813891650224
-0.0959858434950269 0.49118883490947
-0.0685614584169931 0.495482859937622
-0.0411370733389593 0.495384811424993
-0.0137126882609251 0.490726066901461
0.0137116968171087 0.481494353605767
0.0411360818951425 0.467838258541599
0.0685604669731763 0.450061892704693
0.0959848520512105 0.428610020225552
0.123409237129244 0.404044787075039
0.150833622207278 0.377015890812561
0.178258007285312 0.348226546545014
0.205682392363346 0.318397873820888
0.23310677744138 0.288234334773946
0.260531162519414 0.258392608135284
0.287955547597448 0.229455828582273
0.315379932675482 0.201914519957251
0.342804317753516 0.176154879797609
0.37022870283155 0.152454408220213
0.397653087909584 0.130984284662515
0.425077472987617 0.11181743312305
0.452501858065652 0.0949409113193531
0.479926243143685 0.080271120569592
0.507350628221719 0.067670350222345
0.534775013299753 0.0569633165553445
0.562199398377787 0.0479525948817653
0.589623783455821 0.0404321350585353
0.617048168533855 0.0341983563483471
0.644472553611889 0.0290586052616937
0.671896938689923 0.024837005766844
0.699321323767957 0.0213779207926232
0.726745708845991 0.0185473720201063
0.754170093924025 0.0162328339126562
0.781594479002059 0.0143418356334875
0.809018864080092 0.0127997820849495
0.836443249158127 0.0115473551142677
0.86386763423616 0.0105377898315517
0.891292019314194 0.00973424919715123
0.918716404392228 0.00910745047895153
0.946140789470262 0.0086336352826959
0.973565174548296 0.00829292370320387
1.00098955962633 0.00806805385138474
1.02841394470436 0.00794348023097773
1.0558383297824 0.00790478685958873
1.08326271486043 0.0079383618372979
1.11068709993847 0.00803127730782333
1.1381114850165 0.0081713205846103
1.16553587009453 0.00834712702840542
1.19296025517257 0.00854837176483028
1.2203846402506 0.00876598451082913
1.24780902532864 0.00899235887684828
1.27523341040667 0.00922153397009669
1.3026577954847 0.00944933156030286
1.33008218056274 0.00967343626850357
1.35750656564077 0.00989340917005933
1.38493095071881 0.0101106270386368
1.41235533579684 0.0103281405867439
1.43977972087487 0.0105504460633636
1.46720410595291 0.0107831661624822
1.49462849103094 0.0110326391068462
1.52205287610897 0.0113054195952499
1.54947726118701 0.0116077023388141
1.57690164626504 0.0119446880210016
1.60432603134308 0.0123199220265368
1.63175041642111 0.0127346469927574
1.65917480149914 0.0131872195113177
1.68659918657718 0.0136726472970028
1.71402357165521 0.01418230406507
1.74144795673325 0.0147038738466151
1.76887234181128 0.0152215638735687
1.79629672688931 0.0157166057925244
1.82372111196735 0.0161680401957895
1.85114549704538 0.0165537516633196
1.87856988212342 0.0168516938189287
1.90599426720145 0.017041219802867
1.93341865227948 0.0171044163954285
1.96084303735752 0.0170273324825743
1.98826742243555 0.0168009962216561
2.01569180751359 0.0164221303276599
2.04311619259162 0.015893500046359
2.07054057766965 0.0152238609388913
2.09796496274769 0.0144275098907543
2.12538934782572 0.0135234785935195
2.15281373290376 0.0125344400297358
2.18023811798179 0.0114854217704265
2.20766250305982 0.0104024328267322
2.23508688813786 0.00931111242266371
2.26251127321589 0.00823549985924942
2.28993565829392 0.00719700639725117
2.31736004337196 0.00621364553240396
2.34478442844999 0.00529955042812656
2.37220881352803 0.00446477990137078
2.39963319860606 0.00371539012996202
2.42705758368409 0.00305373034314841
2.45448196876213 0.00247890843809989
2.48190635384016 0.001987367052984
2.5093307389182 0.00157351159384486
2.53675512399623 0.00123033788384462
2.56417950907426 0.000950016916332036
2.5916038941523 0.000724405960017567
2.61902827923033 0.000545467405625704
2.64645266430837 0.000405587959898316
2.6738770493864 0.000297800166569805
2.70130143446443 0.000215915255507122
2.72872581954247 0.000154580857885051
};
\addlegendentry{20\%}
\addplot [semithick, crimson2143940]
table {%
-2.65325261860717 0.0105584273944378
-2.62658678479386 0.0151545789975895
-2.59992095098055 0.0214294917105401
-2.57325511716724 0.0298542783391358
-2.54658928335393 0.0409762080044666
-2.51992344954062 0.0554104493534736
-2.49325761572731 0.0738228954483942
-2.466591781914 0.0969025878951038
-2.43992594810069 0.125322975060452
-2.41326011428738 0.159692306730205
-2.38659428047407 0.200494839702246
-2.35992844666076 0.24802609413839
-2.33326261284745 0.302326981527005
-2.30659677903414 0.363122991289803
-2.27993094522082 0.429775517930172
-2.25326511140751 0.501252591194454
-2.2265992775942 0.576125555027435
-2.19993344378089 0.652596549684786
-2.17326760996758 0.728559048092501
-2.14660177615427 0.801690400434905
-2.11993594234096 0.869571711243442
-2.09327010852765 0.929826874735695
-2.06660427471434 0.980269727690373
-2.03993844090103 1.01904650228884
-2.01327260708772 1.04476040725929
-1.98660677327441 1.05656637730137
-1.9599409394611 1.0542267246175
-1.93327510564779 1.038122294567
-1.90660927183448 1.00921827999285
-1.87994343802117 0.968988489089263
-1.85327760420785 0.919305981869797
-1.82661177039454 0.862311068060615
-1.79994593658123 0.800269338376602
-1.77328010276792 0.73543254096021
-1.74661426895461 0.669913800547603
-1.7199484351413 0.605586192501194
-1.69328260132799 0.544010448155056
-1.66661676751468 0.486394063828869
-1.63995093370137 0.433580776788416
-1.61328509988806 0.38606663513781
-1.58661926607475 0.344036972672106
-1.55995343226144 0.3074176042136
-1.53328759844813 0.275933443648278
-1.50662176463482 0.249168368668206
-1.47995593082151 0.226621298468583
-1.4532900970082 0.207754875002686
-1.42662426319488 0.192034621277179
-1.39995842938157 0.178957810523245
-1.37329259556826 0.16807239431346
-1.34662676175495 0.158987142379243
-1.31996092794164 0.151374631888663
-1.29329509412833 0.14496891967113
-1.26662926031502 0.139559693261559
-1.23996342650171 0.134984492934316
-1.2132975926884 0.131120294177049
-1.18663175887509 0.127875397295411
-1.15996592506178 0.125182234570421
-1.13330009124847 0.122991408270252
-1.10663425743516 0.121267034296487
-1.07996842362185 0.119983294442081
-1.05330258980854 0.119121994371795
-1.02663675599523 0.118670877323587
-0.999970922181915 0.118622443920037
-0.973305088368605 0.118973063091774
-0.946639254555294 0.119722214355877
-0.919973420741983 0.120871765026258
-0.893307586928673 0.122425246896248
-0.866641753115362 0.124387147757937
-0.839975919302052 0.126762268956965
-0.813310085488741 0.129555218921732
-0.78664425167543 0.132770114399662
-0.75997841786212 0.136410547669144
-0.733312584048809 0.140479851743013
-0.706646750235498 0.144981659149901
-0.679980916422188 0.149920705680239
-0.653315082608877 0.155303780758445
-0.626649248795566 0.161140673380957
-0.599983414982256 0.167444910396846
-0.573317581168945 0.174234037699814
-0.546651747355635 0.181529162449424
-0.519985913542324 0.189353465927157
-0.493320079729013 0.197729423707328
-0.466654245915703 0.206674543664924
-0.439988412102392 0.216195560820099
-0.413322578289082 0.226281212338411
-0.386656744475771 0.236893947552039
-0.35999091066246 0.247961186117301
-0.33332507684915 0.259366989831671
-0.306659243035839 0.270945218153727
-0.279993409222528 0.28247534794503
-0.253327575409218 0.29368211191674
-0.226661741595907 0.304239918042562
-0.199995907782597 0.313782645578641
-0.173330073969286 0.321918891797846
-0.146664240155975 0.328252116232499
-0.119998406342665 0.332404471338099
-0.0933325725293543 0.334042512057489
-0.0666667387160436 0.332902536693649
-0.0400009049027332 0.328813109299536
-0.0133350710894224 0.321712401934455
0.0133307627238879 0.311658385312382
0.0399965965371987 0.298830554361295
0.0666624303505094 0.283522723855364
0.0933282641638198 0.266127359380095
0.119994097977131 0.247112796315656
0.146659931790441 0.226995425322678
0.173325765603752 0.20630939252067
0.199991599417062 0.185576521026205
0.226657433230373 0.165278999621938
0.253323267043684 0.145836942175683
0.279989100856994 0.127592273793174
0.306654934670305 0.110799645556137
0.333320768483615 0.0956243240523761
0.359986602296926 0.0821463390803319
0.386652436110237 0.070369673449404
0.413318269923547 0.0602349817689895
0.439984103736858 0.0516342365816831
0.466649937550168 0.044425797234848
0.493315771363479 0.0384486356930232
0.51998160517679 0.0335347794407306
0.5466474389901 0.0295193892263893
0.573313272803411 0.0262482304878136
0.599979106616721 0.0235825863622268
0.626644940430032 0.0214018765256383
0.653310774243343 0.0196043832068214
0.679976608056653 0.018106548536814
0.706642441869964 0.0168413085949176
0.733308275683275 0.0157558855724354
0.759974109496585 0.0148093873463239
0.786639943309896 0.0139704785105159
0.813305777123206 0.0132153004791707
0.839971610936517 0.012525739094705
0.866637444749828 0.0118880715071764
0.893303278563138 0.0112919726507208
0.919969112376449 0.0107298262541486
0.946634946189759 0.0101962654713072
0.97330078000307 0.00968786240174203
0.999966613816381 0.009202891715291
1.02663244762969 0.00874110838585047
1.053298281443 0.00830349981394114
1.07996411525631 0.0078919948146339
1.10662994906962 0.00750913266370076
1.13329578288293 0.007157711778365
1.15996161669624 0.00684044771449804
1.18662745050956 0.00655967318418441
1.21329328432287 0.00631710911903817
1.23995911813618 0.00611372684093367
1.26662495194949 0.00594970927100355
1.2932907857628 0.00582450619605293
1.31995661957611 0.00573696712605243
1.34662245338942 0.00568552687626799
1.37328828720273 0.00566841458228259
1.39995412101604 0.00568385648261224
1.42661995482935 0.00573024587341473
1.45328578864266 0.00580625911066925
1.47995162245597 0.00591090323789961
1.50661745626928 0.00604348774222571
1.53328329008259 0.00620351943830982
1.5599491238959 0.00639052533099885
1.58661495770921 0.00660381366260268
1.61328079152252 0.00684218854702259
1.63994662533584 0.00710363891378366
1.66661245914915 0.00738502792485454
1.69327829296246 0.00768181413267577
1.71994412677577 0.00798783949105594
1.74660996058908 0.00829522063902101
1.77327579440239 0.00859437730867132
1.7999416282157 0.0088742242146819
1.82660746202901 0.00912253998794839
1.85327329584232 0.00932650919641155
1.87993912965563 0.00947341289349525
1.90660496346894 0.00955142203268955
1.93327079728225 0.00955042965639539
1.95993663109556 0.0094628452284511
1.98660246490887 0.00928427048467755
2.01326829872218 0.00901398223989363
2.03993413253549 0.00865516370310058
2.06659996634881 0.00821485036502268
2.09326580016212 0.00770358634465559
2.11993163397543 0.00713481817118047
2.14659746778874 0.00652408101707865
2.17326330160205 0.00588805353120221
2.19992913541536 0.00524356889737978
2.22659496922867 0.0046066703517073
2.25326080304198 0.00399178959025725
2.27992663685529 0.00341110825434215
2.3065924706686 0.00287413904222035
2.33325830448191 0.00238753752066067
2.35992413829522 0.00195513185027376
2.38658997210853 0.00157813821778803
2.41325580592184 0.0012555166217958
2.43992163973515 0.000984415489448528
2.46658747354846 0.00076065404631054
2.49325330736177 0.000579197218328321
2.51991914117509 0.000434587400373035
2.5465849749884 0.000321308804262837
2.57325080880171 0.000234071568883558
2.59991664261502 0.000168013008618963
2.62658247642833 0.000118821405390748
2.65324831024164 8.27932184402243e-05
};
\addlegendentry{40\%}
\end{axis}

\end{tikzpicture}

%% file: figures/attackEPS_rpd_M.tex
\begin{tikzpicture}

\definecolor{crimson2143940}{RGB}{214,39,40}
\definecolor{darkgray176}{RGB}{176,176,176}
\definecolor{darkorange25512714}{RGB}{255,127,14}
\definecolor{forestgreen4416044}{RGB}{44,160,44}
\definecolor{lightgray204}{RGB}{204,204,204}
\definecolor{steelblue31119180}{RGB}{31,119,180}

\begin{axis}[
legend cell align={left},
legend columns=2,
legend style={
  fill opacity=0.8,
  draw opacity=1,
  text opacity=1,
  at={(0.03,0.97)},
  anchor=north west,
  draw=lightgray204
},
tick align=outside,
tick pos=left,
x grid style={darkgray176},
xlabel={RPD of M-score},
xmin=-3.19305399281506, xmax=3.19305390043628,
xtick style={color=black},
y grid style={darkgray176},
ylabel={Density},
ymin=0, ymax=1.30530146108764,
ytick style={color=black},
width=\columnwidth,
height=0.5\columnwidth
]
\addplot [semithick, steelblue31119180]
table {%
-2.50564118955376 0.000879632126737964
-2.4804588688273 0.00136331916895249
-2.45527654810084 0.00206641802391053
-2.43009422737438 0.00306312205509058
-2.40491190664792 0.00444055599824964
-2.37972958592145 0.00629562647834252
-2.35454726519499 0.00872912986505289
-2.32936494446853 0.011836831822129
-2.30418262374207 0.015697662728343
-2.27900030301561 0.0203597494898218
-2.25381798228915 0.0258256525384431
-2.22863566156268 0.0320387761277797
-2.20345334083622 0.0388733197087962
-2.17827102010976 0.0461301887780341
-2.1530886993833 0.0535408775410911
-2.12790637865684 0.0607804473271269
-2.10272405793038 0.0674894371551621
-2.07754173720392 0.0733030507504882
-2.05235941647745 0.0778845451749367
-2.02717709575099 0.0809587040472755
-2.00199477502453 0.0823408709181578
-1.97681245429807 0.0819573878498204
-1.95163013357161 0.0798544111801906
-1.92644781284515 0.0761937731016838
-1.90126549211869 0.0712365052794954
-1.87608317139222 0.0653164622110896
-1.85090085066576 0.0588078309815755
-1.8257185299393 0.052090955311421
-1.80053620921284 0.0455207615563419
-1.77535388848638 0.0394012475033717
-1.75017156775992 0.0339682096496269
-1.72498924703346 0.0293809373853654
-1.69980692630699 0.0257222820823119
-1.67462460558053 0.0230055328823691
-1.64944228485407 0.0211860069326771
-1.62425996412761 0.0201751831835809
-1.59907764340115 0.0198554778905734
-1.57389532267469 0.0200942293957922
-1.54871300194822 0.0207559783590477
-1.52353068122176 0.0217125790016104
-1.4983483604953 0.0228509921636148
-1.47316603976884 0.0240787827796733
-1.44798371904238 0.0253274065206826
-1.42280139831592 0.0265533786887183
-1.39761907758946 0.0277374288877323
-1.37243675686299 0.0288817962809527
-1.34725443613653 0.0300059249838926
-1.32207211541007 0.031140963478488
-1.29688979468361 0.0323236223320476
-1.27170747395715 0.0335900586644598
-1.24652515323069 0.0349704946162938
-1.22134283250423 0.0364852149200874
-1.19616051177776 0.0381424202891374
-1.1709781910513 0.0399381561110202
-1.14579587032484 0.0418582275847992
-1.12061354959838 0.0438817044172152
-1.09543122887192 0.0459853659684696
-1.07024890814546 0.048148289745657
-1.045066587419 0.0503557735162877
-1.01988426669253 0.0526019107078931
-0.994701945966073 0.0548903900331095
-0.969519625239611 0.0572334198585253
-0.94433730451315 0.0596490271573508
-0.919154983786688 0.0621572880346
-0.893972663060227 0.0647762585892718
-0.868790342333765 0.0675184564925129
-0.843608021607304 0.0703886840235703
-0.818425700880842 0.0733837952122042
-0.793243380154381 0.0764947256911623
-0.768061059427919 0.0797107692663073
-0.742878738701458 0.0830257509045281
-0.717696417974997 0.0864454614827193
-0.692514097248535 0.0899955284935415
-0.667331776522073 0.093728831620274
-0.642149455795612 0.0977316507465749
-0.616967135069151 0.102127955740561
-0.591784814342689 0.107081588206258
-0.566602493616227 0.112796495097899
-0.541420172889766 0.119515576838309
-0.516237852163305 0.127519013521657
-0.491055531436843 0.137123031926268
-0.465873210710382 0.148679889242183
-0.44069088998392 0.16257933353618
-0.415508569257459 0.179250978455338
-0.390326248530997 0.199166005228632
-0.365143927804536 0.22283556770454
-0.339961607078074 0.250802481122364
-0.314779286351613 0.283622503613765
-0.289596965625151 0.321832019207361
-0.26441464489869 0.365900349446763
-0.239232324172229 0.416167240489016
-0.214050003445767 0.47276907109281
-0.188867682719306 0.535560572639185
-0.163685361992844 0.604041749501769
-0.138503041266382 0.677301570840154
-0.113320720539921 0.753990269338424
-0.0881383998134595 0.83233033321806
-0.0629560790869981 0.910172456693537
-0.0377737583605366 0.98509717464352
-0.0125914376340752 1.05455640820706
0.0125908830923862 1.11604275003171
0.0377732038188476 1.16726919542188
0.0629555245453091 1.20633922640033
0.0881378452717705 1.23188736162569
0.113320165998232 1.24317363440008
0.138502486724694 1.24012148202107
0.163684807451155 1.22329621227815
0.188867128177617 1.19382920745529
0.214049448904078 1.15329992711935
0.239231769630539 1.10359241518429
0.264414090357001 1.04674472438068
0.289596411083462 0.984808369612409
0.314778731809924 0.919731125607413
0.339961052536385 0.853271144549275
0.365143373262847 0.786944633906192
0.390325693989308 0.72200428665183
0.41550801471577 0.659442093758733
0.440690335442231 0.600008474663095
0.465872656168693 0.54423978061907
0.491054976895154 0.492487745654031
0.516237297621616 0.444946756356483
0.541419618348077 0.401677221147682
0.566601939074538 0.362625290147469
0.591784259801 0.327640369483802
0.616966580527462 0.296492198894825
0.642148901253923 0.268888850582501
0.667331221980385 0.244496141710734
0.692513542706846 0.222957969616502
0.717695863433307 0.203916280670748
0.742878184159769 0.187028974002086
0.76806050488623 0.171984092596746
0.793242825612692 0.158509113724653
0.818425146339153 0.14637487339228
0.843607467065615 0.135394458021773
0.868789787792076 0.125418091008779
0.893972108518538 0.116325499944709
0.919154429244999 0.10801740897838
0.944336749971461 0.100407667907902
0.969519070697922 0.0934171694905661
0.994701391424384 0.0869702158224879
1.01988371215085 0.0809934764081607
1.04506603287731 0.075417222441387
1.07024835360377 0.0701781834785336
1.09543067433023 0.0652231821443757
1.12061299505669 0.0605126603517458
1.14579531578315 0.0560232982001893
1.17097763650961 0.0517491154961793
1.19615995723608 0.0477007034750301
1.22134227796254 0.0439025284559108
1.246524598689 0.0403885478555316
1.27170691941546 0.037196649886777
1.29688924014192 0.0343626386959528
1.32207156086838 0.0319146058604194
1.34725388159484 0.0298685329061615
1.37243620232131 0.0282258463698538
1.39761852304777 0.026973403619111
1.42280084377423 0.0260860513780171
1.44798316450069 0.0255315151588522
1.47316548522715 0.0252770032063238
1.49834780595361 0.0252966004800588
1.52353012668007 0.0255783338058848
1.54871244740654 0.0261297373718943
1.573894768133 0.0269808454256431
1.59907708885946 0.0281837744809199
1.62425940958592 0.0298084061890228
1.64944173031238 0.0319341146760896
1.67462405103884 0.0346379681771002
1.6998063717653 0.0379803421010616
1.72498869249177 0.0419893696356882
1.75017101321823 0.0466460721870172
1.77535333394469 0.0518722826384376
1.80053565467115 0.0575235145875202
1.82571797539761 0.0633886609451372
1.85090029612407 0.0691977792035681
1.87608261685054 0.0746382544312756
1.901264937577 0.0793784234130186
1.92644725830346 0.0830964758129015
1.95162957902992 0.0855113589475311
1.97681189975638 0.0864117466980956
2.00199422048284 0.0856790798022838
2.0271765412093 0.0833013213473038
2.05235886193577 0.0793753291997995
2.07754118266223 0.0740974156579125
2.10272350338869 0.067743433204181
2.12790582411515 0.0606412544251759
2.15308814484161 0.0531395159785944
2.17827046556807 0.0455768014526109
2.20345278629453 0.0382550302543233
2.228635107021 0.031419831955091
2.25381742774746 0.0252493572379473
2.27899974847392 0.0198515896996163
2.30418206920038 0.0152690342755676
2.32936438992684 0.011488847314157
2.3545467106533 0.00845611550787872
2.37972903137976 0.00608805940693485
2.40491135210623 0.00428732924046463
2.43009367283269 0.00295313490750215
2.45527599355915 0.00198956655281173
2.48045831428561 0.00131100412723105
2.50564063501207 0.000844914810535108
};
\addlegendentry{5\%}
\addplot [semithick, darkorange25512714]
table {%
-2.67263460262346 0.000730671879239739
-2.6457739547928 0.00104082701286051
-2.61891330696214 0.00146167146010779
-2.59205265913148 0.00202367478350171
-2.56519201130083 0.00276220976514488
-2.53833136347017 0.00371709320653374
-2.51147071563951 0.00493159845771421
-2.48461006780885 0.00645084891775983
-2.4577494199782 0.00831954082803516
-2.43088877214754 0.0105790025748581
-2.40402812431688 0.0132636743610362
-2.37716747648622 0.0163971809337465
-2.35030682865557 0.019988261914122
-2.32344618082491 0.0240269069683708
-2.29658553299425 0.0284811026408045
-2.26972488516359 0.0332946203710618
-2.24286423733293 0.0383862498420287
-2.21600358950228 0.0436508021619246
-2.18914294167162 0.0489620743312346
-2.16228229384096 0.0541777890375056
-2.1354216460103 0.0591463191887569
-2.10856099817965 0.0637147983647501
-2.08170035034899 0.0677380337034504
-2.05483970251833 0.071087503417538
-2.02797905468767 0.0736596592189209
-2.00111840685702 0.0753827778818803
-1.97425775902636 0.0762217181239325
-1.9473971111957 0.07618012872528
-1.92053646336504 0.0752998998395345
-1.89367581553438 0.0736579219047107
-1.86681516770373 0.0713604815875159
-1.83995451987307 0.0685358489825935
-1.81309387204241 0.0653257679274378
-1.78623322421175 0.061876634554344
-1.7593725763811 0.0583311326855994
-1.73251192855044 0.0548209950354508
-1.70565128071978 0.0514613934704959
-1.67879063288912 0.0483472541856147
-1.65192998505847 0.0455515723812072
-1.62506933722781 0.0431255930637765
-1.59820868939715 0.0411005529617057
-1.57134804156649 0.0394905595536175
-1.54448739373584 0.0382961250778025
-1.51762674590518 0.0375078761712675
-1.49076609807452 0.0371100162833993
-1.46390545024386 0.037083215517982
-1.4370448024132 0.037406725020921
-1.41018415458255 0.0380596433026146
-1.38332350675189 0.0390213838220086
-1.35646285892123 0.0402714933515477
-1.32960221109057 0.0417890396300489
-1.30274156325992 0.0435518197540018
-1.27588091542926 0.0455356375338795
-1.2490202675986 0.0477138628172578
-1.22215961976794 0.0500574261399478
-1.19529897193729 0.0525353278169136
-1.16843832410663 0.0551156623824234
-1.14157767627597 0.0577670872577143
-1.11471702844531 0.0604606070382495
-1.08785638061465 0.063171507570013
-1.060995732784 0.0658812596048817
-1.03413508495334 0.0685792197136822
-1.00727443712268 0.07126398298883
-0.980413789292024 0.0739442826312526
-0.953553141461366 0.0766393795592115
-0.926692493630709 0.0793789345011914
-0.899831845800051 0.0822024003356352
-0.872971197969393 0.0851580098893184
-0.846110550138736 0.0883014618875579
-0.819249902308078 0.0916944247390055
-0.79238925447742 0.0954029849970294
-0.765528606646763 0.0994961659423006
-0.738667958816105 0.104044633067592
-0.711807310985447 0.109119688145441
-0.68494666315479 0.114792632139173
-0.658086015324132 0.121134548913252
-0.631225367493474 0.128216525574765
-0.604364719662817 0.136110280563134
-0.577504071832159 0.144889117352459
-0.550643424001501 0.154629061378229
-0.523782776170844 0.165409974104042
-0.496922128340186 0.177316376861919
-0.470061480509528 0.190437666315979
-0.443200832678871 0.204867372920655
-0.416340184848213 0.220701114131104
-0.389479537017555 0.238032935277476
-0.362618889186898 0.256949820425234
-0.33575824135624 0.27752429637025
-0.308897593525582 0.299805242185267
-0.282036945694925 0.323807243985017
-0.255176297864267 0.349499081344029
-0.228315650033609 0.376792172237347
-0.201455002202952 0.405530006242553
-0.174594354372294 0.435479727764342
-0.147733706541636 0.466327061361952
-0.120873058710979 0.497675676473576
-0.0940124108803211 0.52905185786296
-0.0671517630496634 0.55991498603104
-0.0402911152190057 0.589673861790187
-0.0134304673883481 0.617708371832258
0.0134301804423096 0.643395442165208
0.0402908282729673 0.666137726545085
0.0671514761036249 0.685393090383271
0.0940121239342826 0.700702731234006
0.12087277176494 0.711715762013062
0.147733419595598 0.718208286307446
0.174594067426256 0.72009540355658
0.201454715256913 0.717435156398078
0.228315363087571 0.71042411224331
0.255176010918229 0.699384981445146
0.282036658748886 0.684747336624616
0.308897306579544 0.667023039988761
0.335757954410202 0.646778352327596
0.362618602240859 0.62460485676208
0.389479250071517 0.601091276972404
0.416339897902175 0.576798024451179
0.443200545732832 0.552235914866299
0.47006119356349 0.527850006783446
0.496921841394148 0.504008999127472
0.523782489224805 0.481000135708876
0.550643137055463 0.459029153732845
0.577503784886121 0.438224510150862
0.604364432716778 0.418644938499913
0.631225080547436 0.400289325678999
0.658085728378094 0.383107935180022
0.684946376208751 0.367014113923916
0.711807024039409 0.351895773980831
0.738667671870067 0.337626109762265
0.765528319700724 0.324073173116391
0.792388967531382 0.311108068365577
0.81924961536204 0.298611640213383
0.846110263192697 0.286479610263143
0.872970911023355 0.274626178174368
0.899831558854013 0.262986148985987
0.92669220668467 0.251515686278857
0.953552854515328 0.240191826842868
0.980413502345986 0.22901092839926
1.00727415017664 0.21798625660775
1.0341347980073 0.207144947648926
1.06099544583796 0.196524603691719
1.08785609366862 0.186169786479613
1.11471674149927 0.176128666515878
1.14157738932993 0.166450061368064
1.16843803716059 0.157181057947389
1.19529868499125 0.14836536324615
1.2221593328219 0.140042469593881
1.24901998065256 0.132247657489564
1.27588062848322 0.125012794166348
1.30274127631388 0.118367820988541
1.32960192414454 0.112342758708288
1.35646257197519 0.106969997812809
1.38332321980585 0.102286584066519
1.41018386763651 0.0983361612522092
1.43704451546717 0.0951702008161705
1.46390516329782 0.0928481406084592
1.49076581112848 0.0914360823556979
1.51762645895914 0.0910037694514892
1.5444871067898 0.0916196895813402
1.57134775462045 0.0933443213888529
1.59820840245111 0.096221763555756
1.62506905028177 0.100270231526455
1.65192969811243 0.10547215539193
1.67879034594309 0.11176482829317
1.70565099377374 0.119032700503477
1.7325116416044 0.127102454296952
1.75937228943506 0.135741901331394
1.78623293726572 0.144663504850648
1.81309358509637 0.153532950629204
1.83995423292703 0.161982702523553
1.86681488075769 0.169629931044023
1.89367552858835 0.176097662654545
1.920536176419 0.181037536839589
1.94739682424966 0.184152246747891
1.97425747208032 0.185215631360964
2.00111811991098 0.1840885111331
2.02797876774163 0.180728711624988
2.05483941557229 0.175194264480379
2.08170006340295 0.167639447142957
2.10856071123361 0.158304037610969
2.13542135906427 0.147496827405394
2.16228200689492 0.13557497139959
2.18914265472558 0.122921093964924
2.21600330255624 0.109920182914087
2.2428639503869 0.0969381849054187
2.26972459821755 0.0843038985981753
2.29658524604821 0.0722952998926495
2.32344589387887 0.0611308965563737
2.35030654170953 0.050966169892272
2.37716718954018 0.0418946833966251
2.40402783737084 0.0339530708363653
2.4308884852015 0.0271288860322349
2.45774913303216 0.0213702091578688
2.48460978086282 0.0165959456333989
2.51147042869347 0.0127058957675733
2.53833107652413 0.00958988065712151
2.56519172435479 0.00713544575125273
2.59205237218545 0.00523389529451841
2.6189130200161 0.00378461361754786
2.64577366784676 0.00269778733890464
2.67263431567742 0.00189574960111107
};
\addlegendentry{10\%}
\addplot [semithick, forestgreen4416044]
table {%
-2.82824853205167 0.000617100648779224
-2.79982392448108 0.000837176167517538
-2.77139931691051 0.00112387505177341
-2.74297470933992 0.00149301297946124
-2.71455010176934 0.00196271858135572
-2.68612549419876 0.00255332694308177
-2.65770088662818 0.00328711129387224
-2.6292762790576 0.0041878228463213
-2.60085167148702 0.00528001571673544
-2.57242706391644 0.00658814503701396
-2.54400245634586 0.00813544178073847
-2.51557784877528 0.00994258698850017
-2.4871532412047 0.0120262299434826
-2.45872863363412 0.0143974177910656
-2.43030402606354 0.0170600259738514
-2.40187941849296 0.020009297161135
-2.37345481092238 0.023230608458206
-2.3450302033518 0.026698590133252
-2.31660559578122 0.0303767119659656
-2.28818098821064 0.0342174345272099
-2.25975638064006 0.0381629923054699
-2.23133177306948 0.0421468349763761
-2.2029071654989 0.0460957049750353
-2.17448255792832 0.0499322777720066
-2.14605795035774 0.0535782406269932
-2.11763334278716 0.0569576412354162
-2.08920873521658 0.0600003045289595
-2.060784127646 0.0626450980702033
-2.03235952007542 0.0648428267561571
-2.00393491250484 0.0665585568908691
-1.97551030493426 0.0677732070718945
-1.94708569736368 0.0684842957478641
-1.9186610897931 0.0687057980333252
-1.89023648222252 0.0684671315058142
-1.86181187465194 0.067811355824179
-1.83338726708136 0.0667927278760915
-1.80496265951078 0.0654737974980049
-1.7765380519402 0.0639222548531489
-1.74811344436962 0.0622077474701329
-1.71968883679904 0.0603988729809373
-1.69126422922846 0.0585605249629123
-1.66283962165788 0.0567517278132729
-1.6344150140873 0.0550240471844038
-1.60599040651672 0.0534206105609829
-1.57756579894614 0.0519757232566656
-1.54914119137556 0.0507150228798474
-1.52071658380498 0.0496560834000118
-1.4922919762344 0.0488093601254631
-1.46386736866382 0.0481793594660123
-1.43544276109324 0.0477659212201762
-1.40701815352266 0.0475655141350824
-1.37859354595208 0.0475724647832364
-1.3501689383815 0.0477800622590152
-1.32174433081092 0.0481815038326576
-1.29331972324034 0.0487706670020941
-1.26489511566976 0.0495427095608581
-1.23647050809918 0.0504945103963342
-1.2080459005286 0.051624969624228
-1.17962129295802 0.0529351879595761
-1.15119668538744 0.0544285430749986
-1.12277207781686 0.0561106765718575
-1.09434747024628 0.0579894006299269
-1.0659228626757 0.0600745297919272
-1.03749825510512 0.0623776417276027
-1.00907364753454 0.064911771789914
-0.980649039963958 0.0676910497816765
-0.952224432393378 0.0707302931379133
-0.923799824822798 0.0740445777956636
-0.895375217252218 0.0776488151430997
-0.866950609681638 0.0815573692208358
-0.838526002111057 0.0857837514091208
-0.810101394540477 0.0903404289886176
-0.781676786969897 0.0952387783887837
-0.753252179399317 0.100489203323961
-0.724827571828737 0.106101422658109
-0.696402964258157 0.112084913678601
-0.667978356687577 0.118449475063422
-0.639553749116997 0.125205852280842
-0.611129141546416 0.132366348897327
-0.582704533975837 0.139945332824551
-0.554279926405256 0.147959539265353
-0.525855318834676 0.156428073893195
-0.497430711264096 0.165372031723437
-0.469006103693516 0.174813669315362
-0.440581496122936 0.18477509936277
-0.412156888552356 0.195276515237793
-0.383732280981776 0.206333995502756
-0.355307673411196 0.217956980937415
-0.326883065840615 0.230145555061976
-0.298458458270035 0.242887689438261
-0.270033850699455 0.256156633807509
-0.241609243128875 0.26990863605322
-0.213184635558295 0.284081167136779
-0.184760027987715 0.298591802167742
-0.156335420417135 0.313337872770122
-0.127910812846554 0.328196961278152
-0.0994862052759746 0.343028258265373
-0.0710615977053943 0.357674756003128
-0.0426369901348145 0.371966205905541
-0.0142123825642342 0.385722731284152
0.0142122250063461 0.398758960030719
0.042636832576926 0.410888525968511
0.0710614401475063 0.421928781941771
0.0994860477180861 0.431705570444682
0.127910655288666 0.440057906142662
0.156335262859247 0.446842436172365
0.184759870429827 0.451937556065662
0.213184478000407 0.455247069720359
0.241609085570987 0.456703290289116
0.270033693141567 0.456269485593923
0.298458300712147 0.453941578151767
0.326882908282727 0.449749018284063
0.355307515853307 0.443754761482597
0.383732123423887 0.43605430032462
0.412156730994468 0.426773728077369
0.440581338565047 0.416066845838507
0.469005946135628 0.40411136639205
0.497430553706208 0.391104313365283
0.525855161276788 0.377256760139402
0.554279768847368 0.362788095072628
0.582704376417948 0.347920033663323
0.611128983988528 0.332870620564255
0.639553591559108 0.317848472151562
0.667978199129688 0.303047502359613
0.696402806700269 0.288642351009446
0.724827414270849 0.274784696724245
0.753252021841429 0.261600588885317
0.781676629412009 0.249188878985523
0.810101236982589 0.237620775642041
0.838525844553169 0.226940493784551
0.866950452123749 0.217166920904072
0.89537505969433 0.208296184545428
0.923799667264909 0.200304977083668
0.95222427483549 0.193154476614068
0.98064888240607 0.186794695691024
1.00907348997665 0.181169090866712
1.03749809754723 0.176219272998182
1.06592270511781 0.171889668249843
1.09434731268839 0.168131989771427
1.12277192025897 0.164909387713465
1.15119652782955 0.162200148798428
1.17962113540013 0.160000815331082
1.20804574297071 0.15832858773551
1.23647035054129 0.157222866179036
1.26489495811187 0.156745778601427
1.29331956568245 0.15698153865148
1.32174417325303 0.158034482574493
1.35016878082361 0.16002565420862
1.37859338839419 0.163087846758484
1.40701799596477 0.167359072550115
1.43544260353535 0.172974519113749
1.46386721110593 0.180057160424037
1.49229181867651 0.188707321219993
1.52071642624709 0.198991631543691
1.54914103381767 0.210931945896805
1.57756564138825 0.224494921745951
1.60599024895883 0.239583038918808
1.63441485652941 0.256027878287231
1.66283946409999 0.273586450885543
1.69126407167057 0.291941267652439
1.71968867924115 0.31070466229255
1.74811328681173 0.329427630418998
1.77653789438231 0.347613141045657
1.80496250195289 0.364733533938014
1.83338710952347 0.380251267441313
1.86181171709405 0.393641959501036
1.89023632466463 0.404418403641555
1.91866093223521 0.412154072302924
1.94708553980579 0.416504565509176
1.97551014737637 0.417225536112605
2.00393475494695 0.414185823784802
2.03235936251753 0.407374844959069
2.06078397008811 0.396903688992808
2.08920857765869 0.382999825900542
2.11763318522927 0.365995796077494
2.14605779279985 0.346312684174442
2.17448240037043 0.324439538021502
2.20290700794101 0.300910147650251
2.23133161551159 0.276278728936524
2.25975622308217 0.251096054477093
2.28818083065276 0.22588744791472
2.31660543822333 0.201133826075
2.34503004579391 0.177256664322151
2.3734546533645 0.154607408370521
2.40187926093508 0.133461495693349
2.43030386850566 0.114016814378464
2.45872847607624 0.096396143674433
2.48715308364682 0.0806529073309164
2.5155776912174 0.0667794377892019
2.54400229878798 0.0547168969334001
2.57242690635856 0.0443660202306222
2.60085151392914 0.0355979324857425
2.62927612149972 0.0282644082855709
2.6577007290703 0.0222071003878616
2.68612533664088 0.0172654175129747
2.71454994421146 0.0132828843630916
2.74297455178204 0.0101119498838591
2.77139915935262 0.00761731743433197
2.7998237669232 0.00567794918977609
2.82824837449378 0.00418794673478846
};
\addlegendentry{20\%}
\addplot [semithick, crimson2143940]
table {%
-2.90277636130364 0.000619876902589086
-2.87360272999596 0.000826550793251908
-2.84442909868828 0.00109192973745317
-2.81525546738059 0.00142917190796216
-2.78608183607291 0.00185328762674268
-2.75690820476523 0.00238108074608336
-2.72773457345755 0.00303098088056689
-2.69856094214987 0.00382274575293882
-2.66938731084219 0.00477701657060549
-2.64021367953451 0.00591471543932107
-2.61104004822683 0.00725628243630417
-2.58186641691915 0.00882076095798006
-2.55269278561147 0.0106247528960331
-2.52351915430379 0.0126812793432565
-2.4943455229961 0.014998596858073
-2.46517189168842 0.0175790325497286
-2.43599826038074 0.0204179119569901
-2.40682462907306 0.0235026604267588
-2.37765099776538 0.0268121601290196
-2.3484773664577 0.0303164399422302
-2.31930373515002 0.0339767636404855
-2.29013010384234 0.0377461631379275
-2.26095647253466 0.0415704387051868
-2.23178284122698 0.045389618473671
-2.2026092099193 0.0491398372292199
-2.17343557861162 0.0527555620029088
-2.14426194730393 0.056172062105977
-2.11508831599625 0.0593279968466433
-2.08591468468857 0.0621679777383368
-2.05674105338089 0.064644955515248
-2.02756742207321 0.0667222868480593
-1.99839379076553 0.068375351454119
-1.96922015945785 0.0695926164307543
-1.94004652815017 0.0703760792219159
-1.91087289684249 0.07074106093888
-1.88169926553481 0.0707153645032512
-1.85252563422713 0.0703378537235673
-1.82335200291944 0.0696565465181735
-1.79417837161176 0.0687263450648556
-1.76500474030408 0.0676065454346713
-1.7358311089964 0.0663582779353416
-1.70665747768872 0.0650420266684621
-1.67748384638104 0.0637153634384273
-1.64831021507336 0.0624310087859173
-1.61913658376568 0.0612353038958225
-1.589962952458 0.0601671442159976
-1.56078932115032 0.0592573917204487
-1.53161568984264 0.0585287506094544
-1.50244205853496 0.0579960632157327
-1.47326842722727 0.0576669607586626
-1.44409479591959 0.0575427884523494
-1.41492116461191 0.0576197167115919
-1.38574753330423 0.0578899495318229
-1.35657390199655 0.0583429467065874
-1.32740027068887 0.0589665871416187
-1.29822663938119 0.0597482146405347
-1.26905300807351 0.0606755236150772
-1.23987937676583 0.061737258734922
-1.21070574545815 0.0629237183018451
-1.18153211415047 0.0642270651163848
-1.15235848284278 0.0656414601296059
-1.1231848515351 0.0671630428864065
-1.09401122022742 0.0687897886046541
-1.06483758891974 0.0705212748585404
-1.03566395761206 0.0723583915624966
-1.00649032630438 0.0743030266820635
-0.9773166949967 0.0763577572419268
-0.948143063689019 0.0785255711326186
-0.918969432381338 0.0808096402410063
-0.889795801073657 0.0832131597738972
-0.860622169765976 0.085739262473909
-0.831448538458295 0.0883910098744907
-0.802274907150614 0.0911714559466135
-0.773101275842933 0.0940837716427435
-0.743927644535253 0.0971314122177409
-0.714754013227572 0.100318303177355
-0.685580381919891 0.103649015748282
-0.65640675061221 0.107128899425322
-0.627233119304529 0.110764137995085
-0.598059487996848 0.11456169697086
-0.568885856689167 0.118529134969925
-0.539712225381487 0.122674259369391
-0.510538594073806 0.127004617441939
-0.481364962766125 0.13152682761017
-0.452191331458444 0.136245770629979
-0.423017700150763 0.141163676265356
-0.393844068843082 0.146279155958845
-0.364670437535402 0.151586244625413
-0.335496806227721 0.1570735235269
-0.30632317492004 0.162723399940492
-0.277149543612359 0.168511617094606
-0.247975912304678 0.174407059177368
-0.218802280996997 0.180371901276566
-0.189628649689316 0.186362133652099
-0.160455018381636 0.19232846512181
-0.131281387073955 0.198217583393485
-0.102107755766274 0.203973723062796
-0.0729341244585928 0.209540467011387
-0.0437604931509119 0.214862686284182
-0.014586861843231 0.219888509096554
0.0145867694644499 0.224571202822132
0.0437604007721308 0.228870854406459
0.0729340320798118 0.232755744694265
0.102107663387493 0.236203329984909
0.131281294695173 0.239200768424237
0.160454926002854 0.241744957757765
0.189628557310535 0.243842082311502
0.218802188618216 0.245506698479476
0.247975819925897 0.24676041717758
0.277149451233578 0.247630266607551
0.306323082541259 0.248146837596186
0.33549671384894 0.248342325578206
0.36467034515662 0.248248587395948
0.393843976464301 0.247895327501321
0.423017607771982 0.247308517421915
0.452191239079663 0.246509135526181
0.481364870387344 0.245512292596364
0.510538501695025 0.244326784145918
0.539712133002705 0.242955084574763
0.568885764310386 0.241393772909379
0.598059395618067 0.239634356672155
0.627233026925748 0.237664440790192
0.656406658233429 0.23546917351023
0.68558028954111 0.233032891794117
0.714753920848791 0.23034088498996
0.743927552156471 0.227381197655231
0.773101183464152 0.224146399787141
0.802274814771833 0.220635264544538
0.831448446079514 0.216854308633019
0.860622077387195 0.212819167394971
0.889795708694876 0.20855579361539
0.918969340002557 0.204101484323562
0.948142971310238 0.199505751635856
0.977316602617918 0.194831060263298
1.0064902339256 0.190153454280336
1.03566386523328 0.185563088118423
1.06483749654096 0.181164661089804
1.09401112784864 0.177077731379717
1.12318475915632 0.173436855557047
1.152358390464 0.170391465397503
1.18153202177168 0.168105358333804
1.21070565307937 0.166755645212037
1.23987928438705 0.166530974047848
1.26905291569473 0.167628836363224
1.29822654700241 0.170251768624549
1.32740017831009 0.174602289850655
1.35657380961777 0.180876470875993
1.38574744092545 0.189256112311491
1.41492107223313 0.19989961559402
1.44409470354081 0.212931760205942
1.47326833484849 0.228432742415706
1.50244196615617 0.246426975797637
1.53161559746385 0.266872287768563
1.56078922877154 0.289650254230075
1.58996286007922 0.314558480680596
1.6191364913869 0.341305648851883
1.64831012269458 0.369510092277347
1.67748375400226 0.39870253645897
1.70665738530994 0.42833344012937
1.73583101661762 0.457785111462361
1.7650046479253 0.48638846234793
1.79417827923298 0.51344392708624
1.82335191054066 0.538245736222794
1.85252554184834 0.560108431596427
1.88169917315603 0.578394264616248
1.91087280446371 0.592539962516359
1.94004643577139 0.602081296587125
1.96922006707907 0.606673952895112
1.99839369838675 0.606109389897129
2.02756732969443 0.600324657592495
2.05674096100211 0.589405527975487
2.08591459230979 0.57358271656936
2.11508822361747 0.553221424176941
2.14426185492515 0.528804859098888
2.17343548623283 0.500912777202158
2.20260911754051 0.470196370052425
2.2317827488482 0.437351017934503
2.26095638015588 0.403088493423014
2.29013001146356 0.368110151832444
2.31930364277124 0.333082487664933
2.34847727407892 0.298616190479568
2.3776509053866 0.26524952544653
2.40682453669428 0.233436522843989
2.43599816800196 0.203540116953201
2.46517179930964 0.175830055728594
2.49434543061732 0.150485130843318
2.523519061925 0.127599069220959
2.55269269323269 0.10718929056773
2.58186632454037 0.0892076720224531
2.61103995584805 0.0735524656715162
2.64021358715573 0.0600805770623512
2.66938721846341 0.0486195193307506
2.69856084977109 0.0389784928814242
2.72773448107877 0.03095818955894
2.75690811238645 0.0243590692248453
2.78608174369413 0.0189879944009569
2.81525537500181 0.0146632269831916
2.84442900630949 0.0112178850196748
2.87360263761718 0.00850202524621969
2.90277626892486 0.00638355911679019
};
\addlegendentry{40\%}
\end{axis}

\end{tikzpicture}

%% file: figures/attackEPS_rpd_eps.tex
\begin{tikzpicture}

\definecolor{crimson2143940}{RGB}{214,39,40}
\definecolor{darkgray176}{RGB}{176,176,176}
\definecolor{darkorange25512714}{RGB}{255,127,14}
\definecolor{forestgreen4416044}{RGB}{44,160,44}
\definecolor{lightgray204}{RGB}{204,204,204}
\definecolor{steelblue31119180}{RGB}{31,119,180}

\begin{axis}[
legend cell align={left},
legend style={
  fill opacity=0.8,
  draw opacity=1,
  text opacity=1,
  at={(0.03,0.97)},
  anchor=north west,
  draw=lightgray204
},
tick align=outside,
tick pos=left,
x grid style={darkgray176},
xlabel={RPD of EPS},
xmin=-3.10530291490498, xmax=3.10530482960328,
xtick style={color=black},
y grid style={darkgray176},
ylabel={Density},
ymin=0, ymax=0.871387323305604,
ytick style={color=black},
width=\columnwidth,
height=0.5\columnwidth
]
\addplot [semithick, steelblue31119180]
table {%
-2.65717496013557 0.000477268981947296
-2.63046962741008 0.00068334380115304
-2.6037642946846 0.000964019780338807
-2.57705896195911 0.00134000322471919
-2.55035362923362 0.00183527777511102
-2.52364829650813 0.00247672044880475
-2.49694296378265 0.00329332734931855
-2.47023763105716 0.00431498638179749
-2.44353229833167 0.00557076584870359
-2.41682696560619 0.00708673430643588
-2.3901216328807 0.00888338650609102
-2.36341630015521 0.0109728177915847
-2.33671096742972 0.0133558573684522
-2.31000563470424 0.0160194295423141
-2.28330030197875 0.0189344504387302
-2.25659496925326 0.0220545755261205
-2.22988963652777 0.025316082682502
-2.20318430380229 0.0286391032348808
-2.1764789710768 0.0319303019229089
-2.14977363835131 0.0350869651906066
-2.12306830562583 0.0380023007290369
-2.09636297290034 0.0405715992817887
-2.06965764017485 0.0426987835474439
-2.04295230744936 0.0443027880651474
-2.01624697472388 0.0453231927147962
-1.98954164199839 0.0457245775833092
-1.9628363092729 0.045499175784624
-1.93613097654742 0.0446675613831349
-1.90942564382193 0.0432773021734629
-1.88272031109644 0.0413997073372693
-1.85601497837095 0.039124982681146
-1.82930964564547 0.0365562490563113
-1.80260431291998 0.0338029668088305
-1.77589898019449 0.0309743332714703
-1.74919364746901 0.0281731829343539
-1.72248831474352 0.0254908305979978
-1.69578298201803 0.0230031719467438
-1.66907764929254 0.02076821198003
-1.64237231656706 0.0188250479764786
-1.61566698384157 0.0171942059261859
-1.58896165111608 0.0158791290903326
-1.56225631839059 0.0148685508450223
-1.53555098566511 0.014139452550075
-1.50884565293962 0.0136603079945461
-1.48214032021413 0.013394343200914
-1.45543498748865 0.0133025865745619
-1.42872965476316 0.0133465416545219
-1.40202432203767 0.0134903755808602
-1.37531898931218 0.0137025744086294
-1.3486136565867 0.013957066528802
-1.32190832386121 0.0142338542459343
-1.29520299113572 0.0145192192183881
-1.26849765841024 0.0148055798706649
-1.24179232568475 0.0150910795108234
-1.21508699295926 0.0153789755791994
-1.18838166023377 0.0156768870624568
-1.16167632750829 0.015995942908217
-1.1349709947828 0.0163498632924068
-1.10826566205731 0.0167540008706569
-1.08156032933182 0.0172243721493963
-1.05485499660634 0.017776719366739
-1.02814966388085 0.0184256583720958
-1.00144433115536 0.0191839840053154
-0.974738998429876 0.0200622166950313
-0.948033665704389 0.021068477900858
-0.921328332978901 0.0222087742518979
-0.894623000253414 0.0234877493528436
-0.867917667527927 0.0249099291349126
-0.84121233480244 0.0264814444646455
-0.814507002076952 0.0282121683219484
-0.787801669351465 0.0301181598207615
-0.761096336625978 0.032224268949979
-0.734391003900491 0.0345667280947392
-0.707685671175004 0.0371955410352504
-0.680980338449517 0.0401764767872786
-0.654275005724029 0.0435924819283193
-0.627569672998542 0.0475443373509398
-0.600864340273055 0.0521504000864811
-0.574159007547568 0.0575452855775343
-0.54745367482208 0.0638773604896603
-0.520748342096593 0.0713049337027614
-0.494043009371106 0.079991059212104
-0.467337676645619 0.0900969069813384
-0.440632343920131 0.101773724389004
-0.413927011194644 0.115153508094837
-0.387221678469157 0.130338636101679
-0.36051634574367 0.147390868441026
-0.333811013018183 0.166320300563044
-0.307105680292695 0.187075026597655
-0.280400347567208 0.209532414015733
-0.253695014841721 0.233492977001992
-0.226989682116234 0.258677833818292
-0.200284349390746 0.28473062022919
-0.17357901666526 0.311224494593595
-0.146873683939772 0.337674514220917
-0.120168351214285 0.363555209234635
-0.0934630184887979 0.388322670036755
-0.0667576857633105 0.411439953132693
-0.0400523530378232 0.432404161770658
-0.0133470203123358 0.450773236526399
0.0133583124131511 0.466190350349478
0.0400636451386385 0.478403876796549
0.0667689778641258 0.48728119669061
0.0934743105896128 0.492815104922905
0.1201796433151 0.495122224914606
0.146884976040587 0.494433560663822
0.173590308766074 0.491078030580189
0.200295641491562 0.485460448904046
0.227000974217049 0.478035877349599
0.253706306942536 0.469282512358685
0.280411639668023 0.459675281879502
0.307116972393511 0.449662110735409
0.333822305118998 0.43964441415238
0.360527637844485 0.429962854362309
0.387232970569972 0.420888815913322
0.413938303295459 0.412621492386213
0.440643636020947 0.405289992439186
0.467348968746434 0.3989595117652
0.494054301471921 0.393640404247925
0.520759634197408 0.389298923508968
0.547464966922896 0.385868479209332
0.574170299648383 0.383260430525501
0.60087563237387 0.381373683392943
0.627580965099357 0.380102627618105
0.654286297824844 0.379343207835205
0.680991630550332 0.378997139888655
0.707696963275819 0.378974444007343
0.734402296001306 0.379194562435818
0.761107628726793 0.379586367117008
0.787812961452281 0.380087355817985
0.814518294177768 0.380642300410562
0.841223626903255 0.381201566883047
0.867928959628742 0.3817192877679
0.894634292354229 0.382151542747116
0.921339625079717 0.382454693836948
0.948044957805204 0.382584022743653
0.974750290530691 0.38249282014072
1.00145562325618 0.382132068660419
1.02816095598167 0.381450834103891
1.05486628870715 0.380397428298627
1.08157162143264 0.378921334179948
1.10827695415813 0.376975797615717
1.13498228688361 0.374520904788777
1.1616876196091 0.371526894226146
1.18839295233459 0.367977412902856
1.21509828506008 0.363872425317126
1.24180361778556 0.359230524543032
1.26850895051105 0.354090468005716
1.29521428323654 0.348511853452113
1.32191961596202 0.342574942700275
1.34862494868751 0.336379711634319
1.375330281413 0.330044237486148
1.40203561413849 0.323702519350033
1.42874094686397 0.317501766485832
1.45544627958946 0.311599094505069
1.48215161231495 0.306157465894599
1.50885694504043 0.301340629560952
1.53556227776592 0.297306786890502
1.56226761049141 0.294200766912617
1.5889729432169 0.292144646382074
1.61567827594238 0.291227000300755
1.64238360866787 0.291491292127377
1.66908894139336 0.292924267642218
1.69579427411885 0.295445542777866
1.72249960684433 0.298899806229234
1.74920493956982 0.303053128036984
1.77591027229531 0.307594727180566
1.80261560502079 0.312145183740056
1.82932093774628 0.316271499631493
1.85602627047177 0.319508670347447
1.88273160319726 0.321386617253128
1.90943693592274 0.321460557929021
1.93614226864823 0.319342279433159
1.96284760137372 0.31472943138212
1.9895529340992 0.30742994495896
2.01625826682469 0.297379035468337
2.04296359955018 0.284646930499507
2.06966893227567 0.269436401666978
2.09637426500115 0.252070243561803
2.12307959772664 0.232969896248131
2.14978493045213 0.212627306786362
2.17649026317762 0.191572755014762
2.2031955959031 0.170341655983887
2.22990092862859 0.149443275431119
2.25660626135408 0.129333888036091
2.28331159407956 0.110396248011573
2.31001692680505 0.092926434104979
2.33672225953054 0.0771282927413136
2.36342759225603 0.0631149416662356
2.39013292498151 0.0509161960369791
2.416838257707 0.0404903901377655
2.44354359043249 0.0317389053666776
2.47024892315797 0.0245217608247787
2.49695425588346 0.0186728341746335
2.52365958860895 0.0140136009050459
2.55036492133444 0.010364650069559
2.57707025405992 0.00755460029709669
2.60377558678541 0.00542636007334976
2.6304809195109 0.00384092499253277
2.65718625223639 0.00267907126097537
};
\addlegendentry{5\%}
\addplot [semithick, darkorange25512714]
table {%
-2.75672432176527 0.000636946317977472
-2.72901852200941 0.00088099639084795
-2.70131272225354 0.00120400701856546
-2.67360692249767 0.00162580829280056
-2.6459011227418 0.00216918751055116
-2.61819532298594 0.00285966082951183
-2.59048952323007 0.00372498123540421
-2.5627837234742 0.00479433649674968
-2.53507792371833 0.0060972054861198
-2.50737212396247 0.00766186398308355
-2.4796663242066 0.00951356153063097
-2.45196052445073 0.0116724276532572
-2.42425472469486 0.0141512060787742
-2.396548924939 0.0169529556655622
-2.36884312518313 0.0200688917159457
-2.34113732542726 0.0234765659786283
-2.31343152567139 0.0271385928105716
-2.28572572591552 0.0310021185234349
-2.25801992615966 0.0349991984376875
-2.23031412640379 0.039048191537234
-2.20260832664792 0.0430562085737561
-2.17490252689205 0.0469225615740713
-2.14719672713619 0.0505430690041898
-2.11949092738032 0.0538149810982229
-2.09178512762445 0.0566422144103892
-2.06407932786858 0.0589405330866974
-2.03637352811272 0.0606422941747569
-2.00866772835685 0.0617003897402683
-1.98096192860098 0.0620910698683487
-1.95325612884511 0.06181541378447
-1.92555032908925 0.0608993234403316
-1.89784452933338 0.0593920341220458
-1.87013872957751 0.0573632574911209
-1.84243292982164 0.0548991814753355
-1.81472713006578 0.0520976376443011
-1.78702133030991 0.0490628020655619
-1.75931553055404 0.0458998158450381
-1.73160973079817 0.0427096964347079
-1.7039039310423 0.0395848640650673
-1.67619813128644 0.0366055362322753
-1.64849233153057 0.0338371560367032
-1.6207865317747 0.0313289271631715
-1.59308073201883 0.0291134388409276
-1.56537493226297 0.0272072861951446
-1.5376691325071 0.0256125307786362
-1.50996333275123 0.0243188059691516
-1.48225753299536 0.0233058529940714
-1.4545517332395 0.0225462740228423
-1.42684593348363 0.0220083057740629
-1.39914013372776 0.0216584461050056
-1.37143433397189 0.0214638023993066
-1.34372853421603 0.0213940697718835
-1.31602273446016 0.0214230853634962
-1.28831693470429 0.021529939443496
-1.26061113494842 0.021699652878457
-1.23290533519256 0.0219234529705543
-1.20519953543669 0.0221986957882718
-1.17749373568082 0.0225284935826184
-1.14978793592495 0.0229211117290754
-1.12208213616909 0.0233892019705097
-1.09437633641322 0.0239489385434006
-1.06667053665735 0.0246191217584433
-1.03896473690148 0.0254203101435303
-1.01125893714561 0.0263740373763729
-0.983553137389747 0.0275021637154326
-0.955847337633879 0.028826403133319
-0.928141537878012 0.0303680565401068
-0.900435738122144 0.0321479681968053
-0.872729938366277 0.0341867067892577
-0.845024138610409 0.0365049551448502
-0.817318338854541 0.0391240740832846
-0.789612539098674 0.0420667875757027
-0.761906739342806 0.0453579196438654
-0.734200939586938 0.0490250997481335
-0.706495139831071 0.0530993442042044
-0.678789340075203 0.0576154176259109
-0.651083540319336 0.0626118813568682
-0.623377740563468 0.0681307457322267
-0.595671940807601 0.0742166597331957
-0.567966141051733 0.0809155946286955
-0.540260341295865 0.0882730065995646
-0.512554541539997 0.0963314958440884
-0.48484874178413 0.10512801476614
-0.457142942028262 0.114690713889874
-0.429437142272395 0.125035549367184
-0.401731342516527 0.136162808507833
-0.37402554276066 0.148053737749552
-0.346319743004792 0.160667478889375
-0.318613943248924 0.173938532130772
-0.290908143493057 0.187774966454099
-0.263202343737189 0.202057586992285
-0.235496543981322 0.216640243822841
-0.207790744225454 0.23135142583731
-0.180084944469586 0.245997227065851
-0.152379144713719 0.260365702280926
-0.124673344957851 0.27423254671601
-0.0969675452019834 0.287367945819016
-0.069261745446116 0.299544351167469
-0.0415559456902481 0.310544855221191
-0.0138501459343807 0.320171768234271
0.0138556538214871 0.328254952866996
0.0415614535773545 0.334659452113755
0.0692672533332224 0.339291958314821
0.0969730530890898 0.342105716710979
0.124678852844958 0.343103534543298
0.152384652600825 0.342338671250715
0.180090452356693 0.339913509270973
0.20779625211256 0.335976038701382
0.235502051868428 0.330714322043038
0.263207851624295 0.324349227005167
0.290913651380163 0.317125816687899
0.318619451136031 0.309303860302015
0.346325250891898 0.301147969504938
0.374031050647766 0.292917873885802
0.401736850403633 0.284859325289178
0.429442650159501 0.277196067982351
0.457148449915369 0.270123235279003
0.484854249671236 0.263802439287847
0.512560049427104 0.258358715533432
0.540265849182972 0.253879374823093
0.567971648938839 0.250414707072249
0.595677448694707 0.247980381484603
0.623383248450574 0.246561299543548
0.651089048206442 0.246116586082791
0.67879484796231 0.246585352909682
0.706500647718177 0.247892841810039
0.734206447474045 0.249956550917803
0.761912247229913 0.252691970672484
0.78961804698578 0.256017601605332
0.817323846741647 0.259858992966255
0.845029646497515 0.264151624026945
0.872735446253383 0.268842542648439
0.900441246009251 0.273890771301824
0.928147045765118 0.27926658176979
0.955852845520986 0.284949819214497
0.983558645276853 0.290927518241954
1.01126444503272 0.297191093752696
1.03897024478859 0.303733405488315
1.06667604454446 0.310545987118038
1.09438184430032 0.317616700250056
1.12208764405619 0.324928024234531
1.14979344381206 0.332456128412087
1.17749924356793 0.340170799406169
1.20520504332379 0.348036217057186
1.23291084307966 0.356012493314416
1.26061664283553 0.364057813193726
1.2883224425914 0.372130949926024
1.31602824234727 0.380193871815243
1.34373404210313 0.388214120442077
1.371439841859 0.396166623282754
1.39914564161487 0.404034613198469
1.42685144137073 0.411809366760038
1.4545572411266 0.419488545887091
1.48226304088247 0.427073033371323
1.50996884063834 0.434562289708818
1.53767464039421 0.441948419225002
1.56538044015007 0.449209305993669
1.59308623990594 0.456301348343862
1.62079203966181 0.463152465246012
1.64849783941768 0.469656147516835
1.67620363917354 0.475667361653897
1.70390943892941 0.481001068445972
1.73161523868528 0.485433983755481
1.75932103844115 0.488709986233396
1.78702683819701 0.490549278669952
1.81473263795288 0.490661060011144
1.84243843770875 0.488759097311831
1.87014423746462 0.48457924101697
1.89785003722048 0.477897644696886
1.92555583697635 0.468548269743921
1.95326163673222 0.456438205409437
1.98096743648809 0.441559429926177
2.00867323624395 0.423995877484811
2.03637903599982 0.403925039072458
2.06408483575569 0.381613777315894
2.09179063551156 0.357408529388058
2.11949643526743 0.331720554712675
2.14720223502329 0.305007303554439
2.17490803477916 0.277751294225907
2.20261383453503 0.250438059542359
2.2303196342909 0.223534743327384
2.25802543404676 0.197470799507017
2.28573123380263 0.172621990854036
2.3134370335585 0.149298536306989
2.34114283331437 0.127737857659013
2.36884863307023 0.108101973136333
2.3965544328261 0.0904792182488319
2.42426023258197 0.0748896762874347
2.45196603233784 0.0612934935367401
2.4796718320937 0.0496011465617399
2.50737763184957 0.0396847176065375
2.53508343160544 0.0313893059801169
2.56278923136131 0.0245438381296856
2.59049503111717 0.0189707133077731
2.61820083087304 0.0144939116415852
2.64590663062891 0.0109453760348313
2.67361243038478 0.00816964232468894
2.70131823014064 0.00602682272260552
2.72902402989651 0.00439414068612858
2.75672982965238 0.00316627090056391
};
\addlegendentry{10\%}
\addplot [semithick, forestgreen4416044]
table {%
-2.79593008368027 0.00064651576048154
-2.76783025355447 0.000883957751019697
-2.73973042342866 0.00119515639773543
-2.71163059330286 0.0015979377950627
-2.68353076317706 0.0021127012792696
-2.65543093305126 0.00276223607911697
-2.62733110292546 0.00357132637162433
-2.59923127279966 0.00456610773454283
-2.57113144267385 0.0057731486026961
-2.54303161254805 0.00721824685361072
-2.51493178242225 0.00892495394995742
-2.48683195229645 0.0109128662747327
-2.45873212217065 0.0131957536798539
-2.43063229204484 0.0157796262546689
-2.40253246191904 0.0186608686212138
-2.37443263179324 0.0218245929514281
-2.34633280166744 0.0252433736104628
-2.31823297154164 0.028876524543387
-2.29013314141583 0.0326700629031906
-2.26203331129003 0.0365574680980662
-2.23393348116423 0.0404612953841138
-2.20583365103843 0.0442956403192847
-2.17773382091263 0.0479693797156383
-2.14963399078683 0.0513900426176244
-2.12153416066102 0.05446809864843
-2.09343433053522 0.0571213982450091
-2.06533450040942 0.0592794664450242
-2.03723467028362 0.0608873438592477
-2.00913484015782 0.0619086876594573
-1.98103501003201 0.0623278913234117
-1.95293517990621 0.0621510510063128
-1.92483534978041 0.0614056925818844
-1.89673551965461 0.0601392684522479
-1.86863568952881 0.0584165279588908
-1.840535859403 0.0563159504824247
-1.8124360292772 0.0539254980704333
-1.7843361991514 0.0513379886963709
-1.7562363690256 0.0486464087110672
-1.7281365388998 0.0459394733130804
-1.70003670877399 0.0432977093539948
-1.67193687864819 0.0407902802746125
-1.64383704852239 0.0384727048325629
-1.61573721839659 0.0363855467042098
-1.58763738827079 0.0345540781152449
-1.55953755814499 0.0329888536311759
-1.53143772801918 0.0316870749809595
-1.50333789789338 0.0306345874366461
-1.47523806776758 0.0298083241784212
-1.44713823764178 0.0291790069349319
-1.41903840751598 0.0287139173501788
-1.39093857739017 0.0283795713834307
-1.36283874726437 0.028144155459845
-1.33473891713857 0.0279796147827009
-1.30663908701277 0.0278633181085985
-1.27853925688697 0.0277792567147833
-1.25043942676116 0.0277187661591201
-1.22233959663536 0.0276807862750739
-1.19423976650956 0.0276716967876202
-1.16613993638376 0.0277047826480339
-1.13804010625796 0.0277993947651476
-1.10994027613216 0.0279798786685385
-1.08184044600635 0.028274346372668
-1.05374061588055 0.0287133659882316
-1.02564078575475 0.0293286400889226
-0.997540955628948 0.0301517380326647
-0.969441125503146 0.0312129397673378
-0.941341295377344 0.0325402394007544
-0.913241465251542 0.0341585461638655
-0.885141635125741 0.0360891084891691
-0.857041804999939 0.0383491739465414
-0.828941974874137 0.0409518840226457
-0.800842144748335 0.0439063886620708
-0.772742314622533 0.0472181517623374
-0.744642484496731 0.0508894062535571
-0.716542654370929 0.0549197069237235
-0.688442824245127 0.0593065217026034
-0.660342994119326 0.0640457985025395
-0.632243163993524 0.0691324455035182
-0.604143333867722 0.0745606681905384
-0.57604350374192 0.0803241163290545
-0.547943673616118 0.0864158078153056
-0.519843843490316 0.092827813020311
-0.491744013364514 0.0995507016586297
-0.463644183238713 0.106572773031011
-0.435544353112911 0.113879108409961
-0.407444522987109 0.121450500206913
-0.379344692861307 0.129262325473285
-0.351244862735505 0.137283440629861
-0.323145032609703 0.145475179753109
-0.295045202483901 0.153790540157671
-0.266945372358099 0.162173636418628
-0.238845542232297 0.170559497434464
-0.210745712106495 0.178874270675433
-0.182645881980694 0.187035883378835
-0.154546051854892 0.194955192121906
-0.12644622172909 0.202537629995005
-0.0983463916032878 0.209685334816654
-0.0702465614774859 0.216299713170874
-0.0421467313516843 0.222284364718572
-0.0140469012258824 0.227548261034037
0.0140529288999196 0.232009045485769
0.0421527590257216 0.235596298148197
0.0702525891515231 0.238254595241944
0.0983524192773251 0.239946188748204
0.126452249403127 0.240653140560737
0.154552079528929 0.240378767711775
0.182651909654731 0.239148290376488
0.210751739780533 0.237008620578591
0.238851569906334 0.234027283410377
0.266951400032136 0.230290519605862
0.295051230157938 0.225900673229176
0.32315106028374 0.220973015729227
0.351250890409542 0.215632192900567
0.379350720535344 0.210008500806012
0.407450550661146 0.20423419857229
0.435550380786948 0.198440050247838
0.46365021091275 0.192752256684236
0.491750041038551 0.187289895483971
0.519849871164353 0.182162937514505
0.547949701290155 0.17747085798403
0.576049531415957 0.173301814128628
0.604149361541759 0.169732324848415
0.632249191667561 0.166827363403142
0.660349021793363 0.164640763997952
0.688448851919165 0.163215846317787
0.716548682044967 0.162586176615261
0.744648512170768 0.162776406267144
0.77274834229657 0.163803154424985
0.800848172422372 0.165675925963289
0.828948002548174 0.168398075309012
0.857047832673976 0.171967837863103
0.885147662799778 0.176379451930639
0.91324749292558 0.181624385251417
0.941347323051382 0.187692662765429
0.969447153177184 0.194574268804206
0.997546983302986 0.202260570946238
1.02564681342879 0.210745688099149
1.05374664355459 0.220027705528555
1.08184647368039 0.230109627378007
1.10994630380619 0.240999954464361
1.13804613393199 0.252712782258711
1.1661459640578 0.265267330179254
1.1942457941836 0.278686836782925
1.2223456243094 0.292996783635953
1.2504454544352 0.308222440891496
1.278545284561 0.324385757588239
1.30664511468681 0.34150164791473
1.33474494481261 0.359573750797237
1.36284477493841 0.378589764999711
1.39094460506421 0.398516487294918
1.41904443519001 0.419294709489642
1.44714426531582 0.44083416324598
1.47524409544162 0.463008740737723
1.50334392556742 0.485652263319468
1.53144375569322 0.508555116154413
1.55954358581902 0.531462107975004
1.58764341594483 0.554071943182528
1.61574324607063 0.576038698034701
1.64384307619643 0.596975663196974
1.67194290632223 0.616461842497517
1.70004273644803 0.634051276949126
1.72814256657383 0.649285194010387
1.75624239669964 0.661706771418879
1.78434222682544 0.670878066749831
1.81244205695124 0.676398418824583
1.84054188707704 0.677923400876994
1.86864171720284 0.675183225858389
1.89674154732865 0.667999398006051
1.92484137745445 0.656298393310536
1.95294120758025 0.640121247684214
1.98104103770605 0.619628137086371
2.00914086783185 0.595097337661031
2.03724069795765 0.56691833294546
2.06534052808346 0.535579256139692
2.09344035820926 0.501649278501692
2.12154018833506 0.465756938501424
2.14964001846086 0.428565711991725
2.17773984858666 0.39074832085189
2.20583967871247 0.3529613475398
2.23393950883827 0.315821660838364
2.26203933896407 0.279885973114468
2.29013916908987 0.245634563654841
2.31823899921567 0.213459847766063
2.34633882934148 0.183660084451036
2.37443865946728 0.15643813457173
2.40253848959308 0.131904840887434
2.43063831971888 0.110086328458531
2.45873814984468 0.090934336301862
2.48683797997048 0.0743385960066827
2.51493781009629 0.0601402672799353
2.54303764022209 0.0481455125723179
2.57113747034789 0.0381384254231715
2.59923730047369 0.029892699036472
2.62733713059949 0.0231816114279129
2.6554369607253 0.0177860917067038
2.6835367908511 0.0135008027623112
2.7116366209769 0.0101383175727148
2.7397364511027 0.0075315732533081
2.7678362812285 0.00553485716032935
2.79593611135431 0.00402361500541279
};
\addlegendentry{20\%}
\addplot [semithick, crimson2143940]
table {%
-2.82300256288188 0.000661494140532869
-2.79463066820619 0.000898075146308508
-2.7662587735305 0.00120636836147534
-2.73788687885481 0.0016033573886005
-2.70951498417912 0.00210846664064977
-2.68114308950343 0.00274342025886156
-2.65277119482774 0.00353191932924992
-2.62439930015206 0.00449910488489539
-2.59602740547637 0.0056707824269389
-2.56765551080068 0.00707239670388613
-2.53928361612499 0.00872776322850433
-2.5109117214493 0.0106575848843408
-2.48253982677361 0.0128778067288032
-2.45416793209792 0.0153978878268844
-2.42579603742224 0.0182190931628143
-2.39742414274655 0.0213329284973813
-2.36905224807086 0.0247198534753148
-2.34068035339517 0.0283484105931904
-2.31230845871948 0.0321748977067704
-2.28393656404379 0.0361436885034421
-2.2555646693681 0.0401882690414828
-2.22719277469241 0.0442330108323215
-2.19882088001673 0.0481956453082103
-2.17044898534104 0.0519903454962847
-2.14207709066535 0.0555312638582915
-2.11370519598966 0.0587363264568865
-2.08533330131397 0.0615310484551671
-2.05696140663828 0.0638521190023256
-2.02858951196259 0.0656505077094867
-2.0002176172869 0.0668938709695534
-1.97184572261122 0.0675680827906674
-1.94347382793553 0.0676777778080895
-1.91510193325984 0.0672458680359176
-1.88673003858415 0.0663120727624388
-1.85835814390846 0.0649305753348507
-1.82998624923277 0.0631669843401921
-1.80161435455708 0.0610948239653918
-1.7732424598814 0.0587918050387205
-1.74487056520571 0.0563361325958614
-1.71649867053002 0.0538030883341767
-1.68812677585433 0.0512620897911854
-1.65975488117864 0.048774377070994
-1.63138298650295 0.0463914181846958
-1.60301109182726 0.0441540617694857
-1.57463919715157 0.0420924070570924
-1.54626730247589 0.0402263105505648
-1.5178954078002 0.0385664106760085
-1.48952351312451 0.0371155278383169
-1.46115161844882 0.0358702883126003
-1.43277972377313 0.0348228252243412
-1.40440782909744 0.0339624262777421
-1.37603593442175 0.0332770228281558
-1.34766403974606 0.0327544449099738
-1.31929214507038 0.0323833984851127
-1.29092025039469 0.0321541513807692
-1.262548355719 0.0320589406522145
-1.23417646104331 0.0320921346991793
-1.20580456636762 0.0322501974382017
-1.17743267169193 0.0325315089951029
-1.14906077701624 0.0329360981833456
-1.12068888234055 0.033465337424617
-1.09231698766487 0.0341216420292161
-1.06394509298918 0.0349082043366733
-1.03557319831349 0.0358287805948276
-1.0072013036378 0.0368875360018359
-0.978829408962111 0.0380889422302521
-0.950457514286423 0.0394377129166742
-0.922085619610734 0.0409387566582197
-0.893713724935046 0.0425971243186892
-0.865341830259357 0.0444179279143583
-0.836969935583668 0.0464062117370794
-0.808598040907979 0.0485667621483552
-0.780226146232291 0.0509038499163593
-0.751854251556602 0.0534209072173511
-0.723482356880913 0.0561201495865686
-0.695110462205224 0.0590021603252847
-0.666738567529536 0.0620654604213272
-0.638366672853847 0.0653060903962984
-0.609994778178158 0.0687172313944389
-0.58162288350247 0.0722888913104591
-0.553250988826781 0.0760076781398384
-0.524879094151093 0.0798566775904358
-0.496507199475404 0.083815446043006
-0.468135304799715 0.0878601239620042
-0.439763410124026 0.0919636695316551
-0.411391515448337 0.0960962081266524
-0.383019620772649 0.100225490440495
-0.35464772609696 0.104317450581789
-0.326275831421272 0.108336854797422
-0.297903936745583 0.112248031050539
-0.269532042069894 0.116015668731759
-0.241160147394206 0.119605675638452
-0.212788252718517 0.122986075572117
-0.184416358042828 0.126127924395364
-0.156044463367139 0.129006215519248
-0.127672568691451 0.131600738364588
-0.0993006740157618 0.133896846524771
-0.0709287793400732 0.135886087519676
-0.0425568846643847 0.137566644524666
-0.0141849899886957 0.138943543392785
0.0141869046869929 0.140028586311329
0.0425587993626815 0.140839986583243
0.07093069403837 0.141401696657886
0.0993025887140591 0.141742442369423
0.127674483389748 0.141894498577515
0.156046378065436 0.141892262946476
0.184418272741125 0.141770703303375
0.212790167416814 0.141563767959004
0.241162062092502 0.141302856082498
0.269533956768191 0.141015445828135
0.29790585144388 0.140723971252301
0.326277746119569 0.140445025643219
0.354649640795257 0.14018894982126
0.383021535470946 0.139959840798685
0.411393430146635 0.139755990715933
0.439765324822323 0.139570740065807
0.468137219498012 0.139393704647661
0.496509114173701 0.139212314004696
0.524881008849389 0.139013581525974
0.553252903525078 0.138786013836024
0.581624798200767 0.138521560103182
0.609996692876456 0.138217500711192
0.638368587552144 0.137878179310588
0.666740482227833 0.137516492296083
0.695112376903522 0.137155064706348
0.72348427157921 0.136827060647577
0.751856166254899 0.136576598608095
0.780228060930588 0.13645876623761
0.808599955606276 0.136539253877651
0.836971850281965 0.136893649735731
0.865343744957654 0.137606460362182
0.893715639633343 0.138769936252544
0.922087534309031 0.140482792282234
0.95045942898472 0.142848914825143
0.978831323660409 0.145976140717911
1.0072032183361 0.149975177116474
1.03557511301179 0.154958705798193
1.06394700768747 0.161040681387916
1.09231890236316 0.168335791958535
1.12069079703885 0.1769590049834
1.14906269171454 0.187025075081832
1.17743458639023 0.198647846592986
1.20580648106592 0.211939148595332
1.23417837574161 0.227007057837863
1.2625502704173 0.24395330150609
1.29092216509298 0.262869591732468
1.31929405976867 0.283832731183655
1.34766595444436 0.306898406127544
1.37603784912005 0.332093689914654
1.40440974379574 0.359408412574585
1.43278163847143 0.388785704494757
1.46115353314712 0.420112183500741
1.48952542782281 0.45320841123773
1.51789732249849 0.487820379995753
1.54626921717418 0.523612887019382
1.57464111184987 0.560165692201219
1.60301300652556 0.596973321617699
1.63138490120125 0.633449263073231
1.65975679587694 0.668935096973308
1.68812869055263 0.702714821286151
1.71650058522831 0.734034277364793
1.744872479904 0.762125187599967
1.77324437457969 0.786232907873774
1.80161626925538 0.805646614677222
1.82998816393107 0.819730327561198
1.85836005860676 0.827952949022018
1.88673195328245 0.829915415956858
1.91510384795814 0.825373118659222
1.94347574263382 0.814251960294587
1.97184763730951 0.796656791745576
2.0002195319852 0.772871438153691
2.02859142666089 0.74335009562167
2.05696332133658 0.708700471037663
2.08533521601227 0.669659611996483
2.11370711068796 0.627063875701657
2.14207900536365 0.581814870808027
2.17045090003933 0.534843441464855
2.19882279471502 0.487073830650704
2.22719468939071 0.439390059426427
2.2555665840664 0.392606305594059
2.28393847874209 0.347442689182657
2.31231037341778 0.304507413176926
2.34068226809347 0.264285711336848
2.36905416276915 0.227135566416706
2.39742605744484 0.193289722566331
2.42579795212053 0.162863157681733
2.45416984679622 0.135864926412409
2.48254174147191 0.112213141784038
2.5109136361476 0.0917518306669694
2.53928553082329 0.0742684633769363
2.56765742549898 0.0595111009494943
2.59602932017466 0.0472043012046079
2.62440121485035 0.0370631515184644
2.65277310952604 0.0288050286675394
2.68114500420173 0.0221589042530417
2.70951689887742 0.0168722030710445
2.73788879355311 0.0127153719114424
2.7662606882288 0.00948442357141894
2.79463258290449 0.00700178595717842
2.82300447758017 0.00511581324171831
};
\addlegendentry{40\%}
\end{axis}

\end{tikzpicture}